\documentclass[11pt]{article}
\usepackage[utf8]{inputenc}
\usepackage[nohead, nomarginpar, margin=1.0in, foot=0.50in]{geometry}
\usepackage{amsmath,amssymb}
\usepackage{amsthm}
\usepackage{relsize}
\usepackage{mathtools}
\usepackage{bbm}
\usepackage{tikz}
\usepackage{wasysym}
\usepackage{comment}
\usepackage{array}
\newcolumntype{P}[1]{>{\centering\arraybackslash}p{#1}}
\usepackage[
backend=biber,
style=alphabetic,
sorting=ynt, maxbibnames=99
]{biblatex}
\newcommand\norm[1]{\left\lVert#1\right\rVert}

\newcommand{\defeq}{\vcentcolon=}

\newtheorem{theorem}{Theorem}[section]

\newtheorem{lemma}[theorem]{Lemma}

\newtheorem{proposition}[theorem]{Proposition}
\newtheorem{definition}[theorem]{Definition}
\newtheorem{remark}[theorem]{Remark}

\newtheorem{assumption}{Assumption}[section]

\DeclareMathOperator*{\argmin}{arg\,min}

\providecommand{\keywords}[1]
{
  \small	
  \textbf{\textit{Keywords---}} #1
}

\begin{document}
\title{On Excess Risk Convergence Rates of Neural Network Classifiers}
\author{Hyunouk~Ko, Namjoon Suh, and Xiaoming~Huo}
\date{}

\maketitle

\begin{abstract} 
    The recent success of neural networks in pattern recognition and classification problems suggests that neural networks possess qualities distinct from other more classical classifiers such as SVMs or boosting classifiers. This paper studies the performance of plug-in classifiers based on neural networks in a binary classification setting as measured by their excess risks. Compared to the typical settings imposed in the literature, we consider a more general scenario that resembles actual practice in two respects: first, the function class to be approximated includes the Barron functions as a proper subset, and second, the neural network classifier constructed is the minimizer of a surrogate loss instead of the $0$-$1$ loss so that gradient descent-based numerical optimizations can be easily applied. While the class of functions we consider is quite large that optimal rates cannot be faster than $n^{-\frac{1}{3}}$, it is a regime in which dimension-free rates are possible and approximation power of neural networks can be taken advantage of. In particular, we analyze the estimation and approximation properties of neural networks to obtain a dimension-free, uniform rate of convergence for the excess risk. Finally, we show that the rate obtained is in fact minimax optimal up to a logarithmic factor, and the minimax lower bound shows the effect of the margin assumption in this regime.
\end{abstract} 


\section{Introduction}\label{section: introduction}
Neural networks have a long history as a class of functions with competitive performance in pattern recognition and classification problems. Theoretically, their approximation capability of various significant classes of functions as well as their statistical properties as nonparametric estimators have been actively studied. More recently, the rise of deep neural networks as a solution to many previously unsolved problems in the computer science community has led to the investigation of their theoretical properties.

Since then, many papers have shown universal consistency properties for a variety of classifiers. Among the most successful were the support vector machines pioneered by \cite{boser1992training} and \cite{cortes1995support} and kernel methods based on function class of reproducing kernel Hilbert space pioneered by \cite{aizerman1964probability}, \cite{vapnik1974theory}, and others. Various types of neural networks such as shallow feed-forward neural networks with sigmoidal activation, polynomial networks, and Kolmogorov-Lorentz networks were also studied. In particular, a one hidden-layer neural network with sigmoidal activation obtained by minimization of empirical $L_1$ error was shown to be universally consistent in \cite{lugosi1995nonparametric}.

A central tool widely used in proving these consistency and convergence rate results is the collection of concentration inequalities in various contexts. Roughly speaking, the excess risk for classification can be bounded by a term involving the suprema of empirical process indexed by the class of candidate functions for decision rule. Arguably the most important of them for our purposes is the Talagrand inequality \cite{talagrand1994sharper} which gives a functional uniform concentration bound. This allows us to apply localized second-order arguments to obtain sharp convergence rates. See \cite{boucheron2013concentration} for details. 

With the help of relatively new techniques from empirical process theory, a series of seminal papers \cite{mammen1999smooth}, \cite{tsybakov2004optimal}, \cite{audibert2007fast} providing convergence rate that holds uniformly over a class of distributions satisfying some regularity assumptions were published. In particular, two types of classifiers were considered: empirical risk minimization (ERM) rules and plug-in rules. While an ERM rule provides a classifier directly based on a decision set, that is a classifier that outputs $1$ if data belongs to the set and $0$ otherwise, a plug-in rule tries to approximate the regression function $E[Y=1|X=x]$ first and outputs $1$ or $0$ based on whether the function value exceeds a given threshold. Accordingly, the set of assumptions on the joint distribution of $(X,Y)$ differ in that while results on the ERM rule apply under the set complexity assumption, those on the plug-in rule apply under the function class complexity assumption. In addition, the provided rates of convergence are shown to be minimax optimal in respective scenarios. Details of relevant results will be provided in Section \ref{section: review of existing works}.

\begingroup

\setlength{\tabcolsep}{8pt} 
\renewcommand{\arraystretch}{1.5} 
\small

\begin{table*}[t]
\centering
\caption{Rough comparison of related papers}\label{first table}

\begin{tabular}{ |P{2.0cm}||P{2.0cm}|P{2.3cm}|P{2.3cm}|P{2.3cm}|P{3.0cm}| }
    \hline
    Paper & Classifier Type & Classifier Class &Loss function & Assumptions &  Convergence Rate\\
    \hline
    \cite{mammen1999smooth}  & ERM  & ERM from set class, finite sieves &  0-1 loss & $S(\rho), M(\alpha)$ & $n^{-\frac{1+\alpha}{2+\alpha+\rho\alpha}}$\\
    \cite{tsybakov2004optimal} &  ERM & Finite sieves & 0-1 loss  & $S(\rho), M(\alpha)$ &  $n^{-\frac{1+\alpha}{2+\alpha+\rho\alpha}}$\\
    \cite{audibert2007fast} & ERM + plug-in & Linear polynomial &0-1 loss & $M(\alpha), D, H(\beta)$ & $n^{-\frac{\beta(1+\alpha)}{2\beta+d}}$\\
    \cite{bartlett2006convexity} & ERM + plug-in & Boosting classifier & convex losses &  $S(V), C, M(\alpha)$ & $n^{-\frac{V+2}{2(V+1)(2-\alpha)}}$\\
    \cite{blanchard2003rate} & ERM + plug-in & Boosting classifier & exponential or logistic &  $S(V), M(\alpha)$ & $n^{\frac{V+2}{2(2-\alpha)(V+1)}}$\\
    \cite{blanchard2003rate} & ERM + plug-in & Boosting based on decision stumps & exponential or logistic & $ B, M(\alpha)$  & $n^{-\frac{2(1+\alpha)}{3(2+\alpha)}}$\\
    \cite{steinwart2008support} & plug-in & Support vector machines & hinge loss & $M(\alpha), G(\beta)$ & $n^{-\frac{2\beta(\alpha+1)}{2\beta(\alpha+2) + 3\alpha + 4}}$\\
    \cite{kim2021fast} & ERM + plug-in & Neural networks & hinge loss &  $S(\beta), M(\alpha)$ & $n^{-\frac{\beta(\alpha+1)}{\beta(\alpha+2) + (d-1)(\alpha+1)}}$ \\
     \hline
     This paper & ERM + plug-in & Neural networks& logistic loss &  $ M(\alpha)$, $BA$ &$n^{-\frac{1+\alpha}{3(2+\alpha)}}$\\
     \hline
\end{tabular}

\end{table*}

\endgroup

This paper also focuses on the binary classification problem. Specifically, we will provide a non-asymptotic uniform rate of convergence for the excess risk $E[g(X)\neq Y] - L^*$ where $L^*$ is the Bayes risk, under several distributional assumptions for a plug-in rule based on feed-forward ReLU neural networks. In Table \ref{first table}, we provide a comparison of our work with other existing works in the literature. The capitalized letters in the assumptions column mean the following: $S(\rho)$: set class complexity assumption where $\rho$ is the entropy parameter, $S(V)$: set class complexity assumption where $V$ is the VC-dimension, $M(\alpha)$: margin condition, $H(\beta)$: H\"{o}lder class assumption of smoothness index $\beta$, $G(\beta)$: geometric noise assumption, $B$: bounded variation assumption, $C$: function class convexity assumption, $D$: density assumption on $X$, $BA$: Barron approximation space assumption. More details on various assumptions will be given in Section \ref{section: review of existing works}.
We analyze a new, more general setting that has not been studied before. 
The major distinguishing points of our analysis are:
\begin{itemize}
\item {We study a regime in which minimax optimal rates are ``slow" (as fast as $n^{-\frac{1}{3}}$). While it is much more general than typical settings in the literature, dimension-free rates are possible, and neural networks have the approximation power to achieve such rates. Indeed, we show that our neural network-based plug-in rule is minimax optimal up to a logarithmic factor.}
\item{We consider a feed-forward ReLU neural network-based plug-in rule that is chosen based on the minimization of the empirical average of logistic loss, which we denote by $\phi$. This contrasts with the work of \cite{kim2021fast} that uses hinge loss.}
\item{We apply state-of-the-art results on the complexity of deep feed-forward ReLU neural network class combined with the refined localization analysis to obtain a sharp convergence rate.}

\end{itemize}

In addition to controlling the estimation error via applications of empirical process techniques, we also need to control the approximation error. While approximation of high-dimensional functions usually suffers from the curse of dimensionality, which describes the phenomenon that the approximation rate deteriorates in the input dimension, it was first shown in \cite{barron1993universal} that for a class of functions whose variation is bounded in a suitable sense, shallow neural networks attain a dimension-free $N^{-1/2}$ rate of convergence in $L_2$ norm where $N$ is the number of weights in the architecture, and \cite{barron1992neural} refined the result so that the same rate holds also for the uniform norm. In fact, this is the key property we will use in defining the class of candidate functions for the regression function.

In this paper, we consider the scenario where the true regression function is locally characterized by elements of the Barron approximation space proposed in \cite{caragea2020neural}. Unlike the classical Barron space in \cite{barron1993universal}, which is actually a subset of the set of Lipschitz continuous functions, this space includes even discontinuous functions and hence is more general. Also, while \cite{caragea2020neural} does explore estimation bounds for Barron approximation space, the setting is rather restrictive in that it assumes a noiseless setting where there is a deterministic function $f$ such that $Y=f(X)$. In contrast, we work with a general class of probability measures defined jointly on $(X,Y)$. We will analyze the performance of an empirical risk minimizer of a surrogate loss that is widely used in practice for its computational and statistical advantages over the $0$-$1$ loss. Specifically, we derive a non-asymptotic uniform rate of convergence over a class of distributions that roughly speaking, can be characterized by the Barron approximation space. 

Finally, we will conclude by providing a minimax lower bound for the proposed class of distributions. The lower bound shows that achieving uniform convergence rate is inherently difficult in the sense that it cannot be better than $n^{-\frac{1}{3}}$.
\subsection{Main Contributions}
In summary, the purpose of this paper is to show a non-asymptotic and uniform rate of convergence for a sequence of classifiers based on neural networks in a binary classification setting when the regression function is locally characterized by the Barron approximation space. The classifier chosen is based on empirical risk minimization of the logistic loss. We combine refined results from classical classification theory with approximation results for neural networks. Specifically,
\begin{itemize}
    \item{We first derive a nonasymptotic bound (Theorem \ref{theo1}) on the approximate excess $\phi$-risk for a function obtained via empirical minimization of $\phi$-loss instead of $0$-$1$ loss, which is how neural networks are trained in practice. In this preliminary result, while we initially put minimal assumption on the distribution, the obtained bound is distribution-dependent.}
    \item{Second, we consider the class of distributions such that (1) the Mammen-Tsybakov noise condition holds, (2) the regression function $\eta$ locally belongs to the Barron approximation space. Then, we obtain a uniform bound on the excess risk (Theorem \ref{main excess bound}) over this class of distributions for the neural network plug-in classifiers using the preliminary result above. Apart from the general difference in distributional assumptions, this work differs from the two works \cite{mammen1999smooth} and \cite{tsybakov2004optimal} in that first, our classifier is a plug-in rule, and second, it provides results for a feasible classifier that can be obtained via available optimization methods based on gradient descent while in those works the classifier may not be feasible or difficult to obtain. Another comparable work is   \cite{kim2021fast} where a fast convergence rate is shown for a neural network classifier that minimizes empirical risk for hinge loss and when the Bayes classifier is characterized by function classes proposed by \cite{petersen2018optimal}, \cite{tsybakov2005square}. The regime they work under is perhaps less interesting because traditional classifiers such as local polynomials and support vector machines also lead to minimax optimal rates there. In fact, the optimal rates are dependent on the input-dimension and the smoothness of the regression function. Our work demonstrates a regime where dimension-free rates are possible without explicit smoothness or regularity conditions. Our work also critically differs from \cite{koltchinskii2006local} and \cite{massart2006risk} in that we include an approximation error analysis. Furthermore, we specifically focus on neural network learning with logistic loss, making use of the sharpest known bounds on VC-dimension while taking advantage of the approximation power of neural networks. A rough but honest summary and comparison of related works with convergence rate results for binary classification are given in Table \ref{first table}.} We would like to warn the reader, however, that it is oftentimes not a good idea to compare convergence rates per se because the assumptions applied in respective works are different, and even a very subtle difference can lead to wildly different convergence behaviors.
    
    \item{Third, we derive a minimax lower bound (Theorem \ref{minimax lower bound theorem}) for the same class of distributions considered in deriving the upper 
 bound. This result shows that the upper bound on the rate achieved with neural network classifiers is indeed minimax optimal upto a logarithmic factor. Closely related are the minimax lower bounds derived in \cite{audibert2007fast} under mild distributional assumptions and H\"{o}lder regression function class setting.} 
\end{itemize}

\subsection{Organization}
The rest of the paper is organized as follows. In Section \ref{section:background}, we provide all the necessary definitions from classification and empirical process theory that we will be using throughout the paper. In Section \ref{section: review of existing works}, we review existing results and discuss how their assumptions and results differ from each other and from our own. In Section \ref{section:main results}, we state our main results culminating in the rate of convergence for excess risk given in Theorem \ref{main excess bound}. In Section \ref{section: proofs}, we give all the technical proofs for results from Section \ref{section:main results}.

\section{Background}\label{section:background}
\subsection{Basic setup}\label{basic setup}
Let $Z = (X, Y)$ be a $S \defeq \mathbb{R}^d\times \{0,1\}$-valued random vector, and suppose we have a sample of size $n$: $\mathcal{D}_n = \{(X_1, Y_1), \dots (X_n, Y_n)\}$ that are i.i.d. with distribution $P$. Denote by $P_X$ the marginal distribution of $X$ and $(P_X)_n$ the empirical distribution based on the $n$ samples for $P_X$, which is a random measure on $\mathbb{R}^d$ based on the $n$ samples.
The goal is to construct a classifier 
\begin{align} \label{classifier formula}
    M_n: \mathbb{R}^d \times \{\mathbb{R}^d \times \{0,1\} \}^n \rightarrow \{0,1 \}
\end{align}
that assigns a label to any given input in $\mathbb{R}^d$ based on $n$ samples from $\mathcal{D}_n$. 
The quality of $M_n$ is measured by the error function $L$, which takes as input the classifier $M_n$ and outputs the following conditional probability:
\begin{align} \label{loss2}
    L(M_n) \defeq P(M_n(X; X_1,Y_1,\dots,X_n,Y_n)\neq Y|\mathcal{D}_n).
\end{align}
Hence, when the randomness of $\mathcal{D}_n$ is integrated over, $E[L(M_n)]$ also becomes a deterministic real number between $0$ and $1$. Then, we are interested in obtaining a provable upper bound on $E[L(M_n)]$ in terms of $n$ when $M_n$ is a function realized by a neural network.

The minimal possible expected error will be denoted by
\begin{align} \label{minimum error}
    L^* \defeq \inf_{g}E[L(g)]
\end{align}
where the infimum is taken over all measurable classifiers, and expectation is taken with respect to all sources of randomness. It is a well-known result that this infimum is actually achieved by a classifier with prior knowledge of $P$.

To describe the classifier achieving \eqref{minimum error}, we define the so-called regression function $\eta:\mathbb{R}^d \rightarrow [0,1]$ as the Borel-measurable function satisfying:
\begin{align}\label{regression function}
    \eta(\boldsymbol{X}) = E[Y=1|\boldsymbol{X}].
\end{align}
That such function $\eta$ exists and when composed with $\boldsymbol{X}$, is a version of the conditional expectation is shown, for example, in \cite[Theorem 9.1.2]{chung2001course} or in a more general setting of a Polish space, \cite[Theorem 10.2.1 and 10.2.2]{dudley2018real}. 
Then, it is a standard result (see, for example, \cite[Section 2.1]{devroye2013probabilistic}) that the infimum is achieved by the classifier
\begin{align}\label{Bayes classifier}
    M^*(\boldsymbol{X}) = \mathbbm{1}_{\{x: \eta(x)\geq 1/2\}}(\boldsymbol{X}) = 
    \begin{cases}
        1, & \text{if } \eta(\boldsymbol{X})\geq 1/2; \\
        0, & \text{otherwise,}
    \end{cases}
\end{align}
which one can construct with prior knowledge of $P$ and is independent of the samples $\mathcal{D}_n$ so that we have
\begin{align*}
    L^* &= E[L(M^*)] \\
    &\stackrel{\text{definition}}{=} E[P(M^*(X)\neq Y|\mathcal{D}_n)] \text{ }\\
    &\stackrel{\text{independence}}{=} P(M^*(X)\neq Y).
\end{align*}

For a given function $f: \mathbb{R}^d \rightarrow \mathbb{R}$, it is convenient to write
\begin{align} \label{eq6}
    p_f(\boldsymbol{x}) \defeq \mathbbm{1}_{\{x:f(x)\geq 0\}}(\boldsymbol{x}),
\end{align}
where for any subset $A\subset\mathbb{R}^d$, $\mathbbm{1}_{A}$ is the indicator function defined by
\begin{align}
    \mathbbm{1}_{A}(x) \defeq
    \begin{cases}
        1, & \text{ if } x\in A;\\
        0, & \text{ otherwise.}
    \end{cases}
\end{align}
Then, we can regard $p_f$ as an estimator of the function $\mathbbm{1}_{\{\eta(x)\geq 1/2\}}$. These types of classifiers are called plug-in rules in the literature. Using this notation, note that \eqref{Bayes classifier} is the plug-in rule, and $p_{\eta-1/2}$ is also called the Bayes classifier.

Now, we formally define the concepts needed to quantify the performance of a classifier.

\begin{definition}[Excess risk]\label{excess risk}
   The excess risk of a classifier $g_n:\mathbb{R}^d \times \{\mathbb{R}^d \times \{0,1\} \}^n \rightarrow \{0,1 \}$ that depends on samples $\mathcal{D}_n$ is defined as
    \begin{align*} 
        \mathcal{E}(g_n) \defeq E[L(g_n)]  - \inf_{h \text{ measurable}}P(h(\boldsymbol{X})\neq Y).
    \end{align*}
\end{definition}

We will use the notation $P(\cdot)$ to denote integration with respect to $P$ or simply $E[\cdot]$ when the measure is clear from the context. Now while the infimum in the right-hand side of the above display is taken with respect to all measurable $h$, for analysis, it will be convenient to first consider the case when the infimum is taken with respect to a given function class $\mathcal{G}$ that we take as the candidate set for estimation. Accordingly, we define the approximate excess risk below: 
\begin{definition}[Approximate excess risk]
     Suppose $g_n:\mathbb{R}^d \times \{\mathbb{R}^d \times \{0,1\} \}^n \rightarrow \{0,1 \} \in \mathcal{G}$ for some class of functions $\mathcal{G}$. We define the approximate excess risk of $g_n$ with respect to $\mathcal{G}$ as
    \begin{align*} 
        \widehat{\mathcal{E}}(g_n) \defeq E[L(g_n)]  - \inf_{h \in \mathcal{G}}P(h(\boldsymbol{X})\neq Y).
    \end{align*}
\end{definition}
Since we only consider plug-in rules, when a classifier $p_f$ is constructed from $f$ via \eqref{eq6}, we also write $\mathcal{E}(f)$ for $\mathcal{E}(p_f)$ and likewise for the approximate excess risk when there is no room for confusion. 

\subsection{Surrogate loss, classification calibration, and excess risks}
Let $\phi: \mathbb{R} \rightarrow [0,\infty)$ be the logistic loss function:
\begin{align} \label{logistic loss}
    \phi(t) \defeq \log\left(1+e^{-t}\right).
\end{align}
Let $\mathcal{G}$ be some class of real-valued measurable functions $g: \mathbb{R}^d\rightarrow \mathbb{R}$ realized by neural networks, and define the function $\phi\bullet g: \mathbb{R}^d\times\{0,1\}\rightarrow \mathbb{R}$,
\begin{align}\label{bullet function}
    \phi \bullet g(x,y) = \phi((2y-1)g(x)).
\end{align}
Based on the above notation, we can analogously define the excess $\phi$-risk and approximate excess $\phi$-risk of a function $g_n$ with respect to function class $\mathcal{G}$ as the approximate excess risk with respect to the class of functions $\{\phi\bullet g, g\in \mathcal{G} \}$ as follows:
\begin{align}\label{excess phi-risk}
    \mathcal{E}_{\phi}(g_n) &\defeq E(\phi \bullet g_n(X,Y; \mathcal{D}_n)) - \inf_{g \text{ measurable}}P(\phi \bullet g(X,Y)),\\
    \widehat{\mathcal{E}}_{\phi}(g_n) &\defeq E(\phi \bullet g_n(X,Y; \mathcal{D}_n)) - \inf_{g \in \mathcal{G}}P(\phi \bullet g(X,Y)). \label{approximate excess phi risk}
\end{align}
We will also write $L_{\phi}(g) \defeq E[\phi\bullet g]$
In our analysis, we will be interested in $\widehat{g}_n$, which is the solution of the following empirical risk minimization problem:
\begin{align*}
    \widehat{g}_n \defeq \argmin_{g\in \mathcal{G}}E_n(\phi\bullet g), 
\end{align*}
where $E_n$ denotes expectation with respect to the empirical measure $\frac{1}{n} \sum_{i=1}^n \delta_{X_i, Y_i}$ based on samples $\mathcal{D}_n$.
We will first study the statistical quality of $\widehat{g}_n\in \mathcal{G}_n$ as measured by its approximate excess $\phi$-risk, and from there derive a convergence rate for the excess risk $\mathcal{E}(\widehat{g}_n)$.

Although in our examination, $\phi$ will always be the logistic loss, statistical properties of a general class of convex losses have been extensively studied, for example in \cite{bartlett2006convexity}. With regard to the use of a surrogate loss, an important concept is that of classification-calibration. We first introduce some preliminary definitions.

\begin{definition}
    Define the optimal conditional $\phi$-risk as
    $$
    H(\eta) \defeq \inf_{\alpha \in \mathbb{R}} \eta \phi(\alpha) + (1-\eta)\phi(-\alpha), \quad \eta\in[0,1].
    $$
\end{definition}
Note that optimal $\phi$-risk is then given by
\begin{align}\label{optimal phi-risk}
    L_{\phi}^* \defeq E[H(\eta(X))] = \inf_{f \textit{ measurable}}E[\phi((2Y-1)f(X))]
\end{align}
where the second equality follows from the property of conditional distribution; see, for example, \cite[Theorem 10.2.1]{dudley2018real}, and that for the logistic loss, the infimum on the right-hand side is achieved by $f(x) = \log(\frac{\eta(x)}{1-\eta(x)})$, which is measurable..
We also define a similar function $H^-$ as
\begin{definition}
    $H^{-}(\eta) \defeq \inf_{\alpha : \alpha(2\eta-1)\leq 0} \eta \phi(\alpha) + (1-\eta)\phi(-\alpha), \quad \eta\in[0,1]$,
\end{definition}
\noindent which is the optimal conditional $\phi$-risk under the constraint that $\alpha$ takes a sign different from $2\eta -1$.
Based on the above definitions, we define classification calibration:
\begin{definition}\label{classification calibrated}
    The surrogate loss $\phi$ is classification calibrated if, for any $\eta \neq 1/2$,
    $$
    H^{-}(\eta)>H(\eta). 
    $$
\end{definition}
Intuitively this says that the $\phi$-risk associated with a ``wrong" classifier should always yield a higher value of $\phi$-risk than the ``correct" one.
It is not trivial to see how the excess $\phi$-risk \eqref{excess phi-risk} relates to the excess risk (Definition \ref{excess risk}), which is of ultimate interest. \cite{zhang2004statistical} first showed the so-called Zhang's inequality: for a function $f_n:\mathbb{R}^d \rightarrow \mathbb{R}$, if $\phi$ is such that for some positive constants $s \geq 1$ and $c \geq 0$
$$
\left|\frac{1}{2}-\eta\right|^s \leq c^s(1-H(\eta)), \quad \eta \in[0,1]
$$
then, we have 
\begin{align} \label{zhang's inequality}
    \mathcal{E}(f_n) \leq c\mathcal{E}_{\phi}(f_n)^{1/s}.
\end{align}
In fact, this result was later refined using the additional assumption of Mammen-Tsybakov noise condition by \cite{bartlett2006convexity} as follows:
\begin{align}\label{bartlett bound}
   \mathcal{E}(f_n) \leq C\mathcal{E}_{\phi}(f_n)^{(1+\alpha)/(s+\alpha)}
\end{align}
where $C>0$ is a constant, $s$ is as in Zhang's inequality, and $\alpha$ is the noise exponent in the Mammen-Tsybakov noise condition, which will be discussed later in Section \ref{dist assumption 3}.
In particular, since the logistic loss is convex and differentiable at $0$ with a negative derivative, it can be shown that it is classification-calibrated so that $\phi$ satisfies the conditions needed to conclude \eqref{bartlett bound}. 
Hence, it suffices to provide a bound of $\mathcal{E}_{\phi}(\widehat{g}_n)$ to bound $\mathcal{E}(\widehat{g}_n)$.

Lastly, following \cite{bartlett2006convexity}, we define the $\psi$-transform of a loss function $\phi$ in the following:
\begin{definition}\label{psi transform}
    Given $\phi:\mathbb{R} \rightarrow [0,\infty)$, define $\psi$-transform 
    $\psi:[-1,1] \rightarrow [0, \infty)$ by $\psi=\Tilde{\psi}^{**}$ where
    \begin{align*}
        \Tilde{\psi}(\eta) = H^-\left(\frac{1+\eta}{2}\right) - H\left(\frac{1+\eta}{2}\right)
    \end{align*}
    and $\Tilde{\psi}^{**}$ is the Fenchel-Legendre biconjugate of $\Tilde{\psi}$. Recall that the Fenchel-Legendre biconjugate of a function $f$ is the closed convex hull of $f$, or equivalently, the largest lower semi-continuous function $f^{**}$ such that $f^{**} \leq f$.
\end{definition}

\subsection{Concepts from empirical risk minimization theory}
In this subsection, we give a number of definitions that frequently appear in empirical risk minimization and empirical process theory used in the proofs of our main results. 
\begin{definition}[$\delta$-minimal set of $P$-risk]
For any class of functions $\mathcal{G}$, we define the $\delta$-minimal set of $P$ for $\mathcal{G}$ as 
\begin{align*}
    \mathcal{G}(\delta) \defeq \{g: g \in \mathcal{G}, E[g] - \inf_{f\in\mathcal{G}} E[f] \leq \delta \}.
\end{align*}
\end{definition}
Since we will be working with $\phi$-risks, it is convenient to define the analogous $\delta$-minimal set of $\phi$-risk as follows:
\begin{definition}[$\delta$-minimal set of $\phi$-risk]
    \begin{align}
        \mathcal{G}_{\phi}(\delta) \defeq \{g: g \in \mathcal{G}, \widehat{\mathcal{E}}_{\phi}(g) \leq \delta \}.
    \end{align}
\end{definition}

We also define VC-dimension widely used as a useful measure of function class complexity. First is the definition of VC-index or VC-dimension for a collection of subsets of a given space.

\begin{definition}[VC-index, VC-class of sets]
    Let $\mathcal{C}$ be a class of subsets of a set $\mathcal{X}$. For an arbitrary subset of $\mathcal{X}$, $\{x_1,\dots,x_n\}$, $\mathcal{C}$ is said to shatter $\{x_1,\dots,x_n\}$ if $|\{C \cap\{x_1,\dots,x_n\}: C\in\mathcal{C}\}| = 2^n$ where $|\cdot|$ denotes cardinality of the set. The VC-index of $\mathcal{C}$ is then defined as
    \begin{align*}
        V(\mathcal{\mathcal{C}}) \defeq \inf \{n: \forall A \subset \mathcal{X} \textit{ with } |A|=n, \mathcal{C} \textit{ does not shatter } A\}.
    \end{align*}
    If $V(\mathcal{C})$ is finite, $\mathcal{C}$ is said to be of VC-class with VC-index $V(\mathcal{C})$.
\end{definition}
Now we can further extend the definition to a collection of functions. By a subgraph of a function, we mean
\begin{definition}[subgraph of function]
    For a function $f:\mathcal{X}\rightarrow \mathbb{R},$ the subgraph of $f$ is the set
    \begin{align*}
        \{(x,t)\in \mathcal{X}\times \mathbb{R}: t < f(x)\}.
    \end{align*}
\end{definition}

\begin{definition}[VC-class of functions]
    A class of functions $\mathcal{F}$ is said to be of VC-class if the set $
    \mathcal{C}_{\mathcal{F}}\defeq\{\textit{subgraph of f}:f\in\mathcal{F}\}$ is of VC-class, and we define the VC-index of $\mathcal{F}$ as $V(\mathcal{F}) = V(\mathcal{C}_{\mathcal{F}})$
\end{definition}
In other words, the VC-index for a class of real-valued functions is defined  in terms of the VC-index of the class of their subgraphs.

\begin{remark}
    Sometimes, the VC-dimension is defined for a class of binary-valued functions and the definition is extended to real-valued functions by transforming them via a fixed threshold function, and the term ``pseudodimension" is used instead to refer to our definition of VC-index for real-valued functions. However, for our purposes, the two definitions can be used interchangeably up to change in constants. See page 2 of \cite{bartlett2019nearly} for details on justification.
\end{remark}
Now that we have defined the VC-class of functions, we introduce the Rademacher process, a special stochastic process widely used in functional concentration results:

\begin{definition}[Rademacher process]
    We define the Rademacher process indexed by a class of functions $\mathcal{F}$ as
    \begin{align}\label{Rademacher process}
        R_n(f) = \frac{1}{n}\sum_{i=1}^n\epsilon_i f(X_i), \quad f\in\mathcal{F},
    \end{align}
    where $\epsilon_i$'s are i.i.d. Rademacher random variables (discrete random variables with mass 1/2 each on -1 and 1).
\end{definition}

Based on $\delta$-minimal set, we can define the expected sup-norm 
\begin{align} \label{expected sup norm}
    \phi_n(\delta) \defeq \phi_n(\mathcal{G},\delta) \defeq E_{\mathcal{D}_n}\left[ \sup_{g_1,g_2 \in \mathcal{G}(\delta)} |(E_n- E)(g_1-g_2)| \right]
\end{align}
where $E_{\mathcal{D}_n}$ denotes expectation with respect to $\mathcal{D}_n$, $E_n$ denotes expectation with respect to empirical measure based on $\mathcal{D}_n$, and $E$ denotes expectation with respect to $(\boldsymbol{X},Y)$. We also define
the $L_2(P)$-diameter of the $\delta$-minimal set $\mathcal{G}(\delta)$ as
\begin{align}\label{diameter}
    D^2(\delta) \defeq D^2(\mathcal{G},\delta) \defeq \sup_{g_1,g_2 \in \mathcal{G}(\delta)}E(g_1-g_2)^2.
\end{align}
We also define some function transformations introduced, for example, in \cite{koltchinskii2011oracle} that will be used in technical parts of later proofs. Given a function $\psi: \mathbb{R}_{\geq 0} \rightarrow \mathbb{R}_{\geq 0}$, define
\begin{align}
    \psi^{b}(\delta) &\defeq \sup_{\sigma \geq \delta} \frac{\psi(\sigma)}{\sigma}, \nonumber\\
    \psi^{\sharp}(\epsilon) &\defeq \inf\{\delta: \delta>0, \psi^{b}(\delta) \leq \epsilon\}\label{transformations}.
\end{align}

\subsection{Neural Network Class}
We consider constructing classifiers through a hybrid of plug-in and ERM procedures as realized by neural networks. Namely, we obtain an estimator $\widehat{f}$ of the Bayes classifier that belongs to the class of feed-forward ReLU networks. 
\begin{figure}[!t]
    \centering
    \includegraphics[height=4cm]{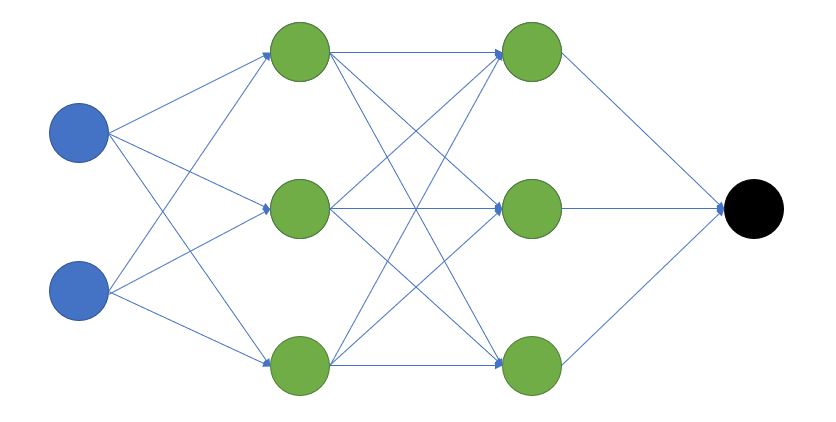}
    \caption{An example of a fully-connected feed-forward neural network with 2 hidden layers, input dimension 2, and output dimension 1.}
    \label{fig:neural net}
\end{figure}

Write the ReLU function
\begin{align*}
    \sigma(x) = \max\{x,0\}, x\in \mathbb{R}.
\end{align*}
Then, a feed-forward ReLU neural network with $L$ hidden-layers and width vector $\boldsymbol{p} = (p_1,\dots,p_L)$ is the function $f$ defined by the following:
\begin{align*}
    f(\boldsymbol{x}) = \sum_{i=1}^{p_L} c_{1,i}^L f_i^{L}(\boldsymbol{x}) + c_{1,0}^L,
\end{align*}
which is the output of the last layer of the neural network, and $f_i^l$ for $l=1,\dots, L, i=1,\dots,p_l$ are recursively defined by
\begin{align*}
    f_i^l(\boldsymbol{x}) =\sigma \left( \sum_{j=1}^{p_{l-1}} c_{i,j}^{l-1} f_j^{l-1}(\boldsymbol{x}) + c_{i,0}^{l-1}\right)
\end{align*}
and for the base case,
\begin{align*}
    f^0_{j}(x) \defeq x_j
\end{align*}
for constants $c_{i,j}^L \in \mathbb{R}$ for all corresponding indices $i,j,l$. We also say that the $l$th hidden layer has $m$ neurons when $p_l=m$. See Figure \ref{fig:neural net} for an illustration.

Then, we denote the class of feed-forward ReLU neural networks by $\mathcal{F}(L,p)$ where
\begin{align*}
    \mathcal{F}(L,p) = \{f: f \textit{ is a feed-forward ReLU neural network with}\\ \textit{$L$ layers and width vector $p$}\}.
\end{align*}

\subsection{Barron approximation space}\label{subsection: Barron approximation space}
In this subsection, we define the target class of functions we wish to approximate. Specifically, we give a characterization of the function class that the regression function $\eta$ belongs to. We introduce several definitions following \cite{caragea2020neural}. 

\begin{definition}[Barron Approximation Space]\label{barron approximation space}
    Let $U \subset \mathbb{R}^d$ be bounded with a non-empty interior and $C>0$ a constant. We define $\mathcal{BA}_C(U)$ to be the set of all measurable functions $f:U \rightarrow \mathbb{R}$ such that for every $m\in \mathbb{N}$, there exists a 1-hidden layer ReLU neural network $g$ with $m$ neurons such that 
    \begin{align}\label{eq38}
        \norm{f - g}_{\infty} \leq \sqrt{d}Cm^{-1/2},
    \end{align}
    and such that all weights involved in $g$ are bounded in absolute value by
    \begin{align}\label{weight bound}
        \sqrt{C} \cdot\left(5+\inf _{x_{0} \in U}\left[\left\|x_{0}\right\|_{1}+\vartheta\left(U, x_{0}\right)\right]\right)
    \end{align}
    where $\vartheta(U, x_{0}) \defeq \sup_{\xi \in \mathbb{R}^{d} \backslash\{0\}}(\|\xi\|_{\infty} /|\xi|_{U, x_{0}}) $ 
    and $|\xi|_{U, x_{0}} \defeq \sup_{x \in U}|\langle\xi, x-x_{0}\rangle|$.
    Then, we define $\mathcal{BA}(U) \defeq \bigcup_{C>0}\mathcal{BA}_C(U)$ and call it the Barron approximation space.
\end{definition}

We give some interpretations of the term \eqref{weight bound}. Since $U$ is open, it is possible to find an open rectangle in $U$ whose edge widths are all greater than some positive $\delta>0$. Then, it is straightforward to verify that independent of the choice of $x_0 \in U$, $\vartheta\left(U, x_{0}\right)$ is bounded from above by $\frac{2}{\delta}$ (c.f. \cite[Remark 2.3]{caragea2020neural}). It is also bounded from below by a positive number since $\|\xi\|_{\infty} /|\xi|_{U, x_{0}}\geq \frac{\norm{\xi}_{\infty}}{\sup_{x\in U}\norm{\xi}\norm{x-x_0}}\geq \frac{1}{\sqrt{d}D}$ where $D$ is the diameter of the set $U$. With this view, we see that \eqref{weight bound} is independent of the choice of $f$ and fully determined by the shape and size of $U$.

The definition of Barron approximation space is motivated by the direct approximation theorems for functions in the classical Barron space as introduced in \cite{barron1993universal}, \cite{barron1992neural}. Barron functions there are defined as functions that admit a Fourier integral representation with finite first moment with respect to the magnitude distribution of the defining complex measure $F$ in the integral representation.
It is worth mentioning that there exist several distinct definitions of ``Barron space" used in the literature.  
Under the definition of \cite{caragea2020neural} (Definition 2.1), it is shown that any Barron function can be estimated by a 1-hidden layer neural network at the same rate as in \eqref{eq38}, which motivates the given definition for Barron approximation space. Hence, Barron approximation space includes all the Barron functions as defined in \cite{barron1993universal}, which includes many important classes of functions such as functions with high-order derivatives. See section V of \cite{barron1993universal} for more examples. For another common definition of Barron space, the convergence rate of $O(m^{-1/2})$ in $L^2$ and $O(\sqrt{d}m^{-1/2})$ in $L^{\infty}$ are shown, see \cite{ma2018priori}, \cite{ma2020towards}. Some of the embedding relationships between these different definitions are investigated in \cite{caragea2020neural}.

\section{Review of existing works} \label{section: review of existing works}

In this subsection, we briefly review some important risk bounds from the literature. In Section \ref{subsection: distributional assumptions}, we first discuss common assumptions on the joint distribution of $(X,Y)$. Section \ref{subsection: Results on ERM classifiers}, we discuss empirical risk-minimizing classifiers and the results associated with them. In Section \ref{subsection: Results on plug-in estimators}, we discuss plug-in classifiers and the results associated with them.

\subsection{Distributional assumptions}\label{subsection: distributional assumptions}
To obtain any meaningful results on the convergence rate of excess risk, assumptions on the distribution governing $(X, Y)$ are necessary. Indeed, Theorem 7.2 of \cite{devroye2013probabilistic} shows that for any sequence of classifiers, there always exists a ``bad" distribution such that the excess risk converges to $0$ at an arbitrarily slow rate. 

In this subsection, we introduce three assumptions on distribution commonly used in the literature. In subsection \ref{dist assumption 1}, we discuss the assumption on the class of candidate sets where the optimal decision set belongs, specifically its set-class complexity. In subsection \ref{dist assumption 2}, we discuss the assumption on function class complexity, which is the complexity of the class of candidate functions that the regression function $\eta$ defined in \eqref{regression function} belongs to. In subsection \ref{dist assumption 3}, we discuss the Mammen-Tsybakov margin assumption, which is an assumption on the behavior of $\eta$ near the decision boundary. Specifically, we impose restrictions on the probability of sets where the regression function is close to the boundary, i.e., points $x$ such that $\eta(x)$ is near $1/2$. In the rest of the paper, we will call this the margin assumption.

\subsubsection{Assumption on the complexity of sets}\label{dist assumption 1}
Let $G^* = \{x: \eta(x) \geq 1/2 \}$, which is the decision set corresponding to the Bayes classifier that satisfies
\begin{align*}
    G^* = \argmin_{G \text{ measurable set}} P(Y \neq \mathbbm{1}(X \in G)). 
\end{align*}
A standard assumption in the ERM framework is that for some constant $\rho>0$, $G^*$ belongs to some given class of sets $\mathcal{G}$ that satisfies
\begin{align}\label{complexity assumption}
    \mathcal{H}(\epsilon, \mathcal{G}, d_{\Delta}) \leq c_0 \epsilon^{-\rho}
\end{align}
where $\mathcal{H}(\epsilon, \mathcal{G}, d_{\Delta})$ denotes the $\epsilon$-entropy of the set $\mathcal{G}$ with respect to the pseudo-metric $d_{\Delta}$ defined as 
\begin{align}
    d_{\Delta}(A,B) \defeq P_X(A\Delta B)
\end{align}
for sets $A$ and $B$ in $\mathbb{R}^d$ and $A\Delta B \defeq (A\backslash B) \bigcup (B\backslash A)$ is the symmetric difference of sets.
Recall that $\epsilon$-entropy $\mathcal{H}(\epsilon, \mathcal{G}, d_{\Delta})$ is defined as the minimum number of $d_{\Delta}$-balls with radius $\epsilon$ required to cover $\mathcal{G}$. 
Note it is necessary that $\mathcal{G}$ be totally bounded with respect to the pseudo-metric $d_{\Delta}$ to satisfy this complexity assumption since otherwise, $\mathcal{H}(\epsilon, \mathcal{G}, d_{\Delta})$ will be infinite.

We also remark that sometimes, a closely-related concept of $\delta$-entropy with bracketing is used, for example in \cite{tsybakov2004optimal}. Since we won't need it for our purpose, we omit the details. 

\subsubsection{Assumption on regression function $\eta$} \label{dist assumption 2}
Similarly, we may assume we are given a large class of functions $\Sigma$ which includes the true $\eta$ we are looking for. Then, a complexity assumption on $\Sigma$ is described as 
\begin{align}\label{car assumption}
    \mathcal{H}(\epsilon,\Sigma,L_p) \leq c_1 \epsilon^{-\rho}
\end{align}
where $\rho>0$, $p \geq 1$, and $\mathcal{H}(\epsilon,\Sigma,L_p)$ is the $\epsilon$-entropy of the class of functions $\Sigma$ with respect to the $L_p$ norm for $P_X$ on $\mathbb{R}^d$. Recall that $\mathcal{H}(\epsilon,\Sigma,L_p)$ is defined as the natural logarithm of the minimal number of $\epsilon$-balls in $L_p$ norm to cover $\Sigma$.

\subsubsection{Margin assumption} \label{dist assumption 3}
This is an assumption on the joint distribution of $X$ and $Y$. Intuitively, it controls (with margin parameter $\alpha$) how $\eta(x)$ behaves around the boundary of the optimal set $\{x:\eta(x)\geq 1/2 \}$. Bigger $\alpha$ means there is a jump of $\eta(x)$ near this boundary, which is favorable for learning, and smaller $\alpha$ close to $0$ means there is a plateau behavior near the boundary, a difficult situation for learning. Specifically, we assume there exist constants $C_0>0$ and $\alpha \geq 0$ such that
\begin{align}\label{margin assumption}
    P_X(0<|\eta(x)-1 / 2| \leq t) \leq C_0 t^\alpha, \quad \forall t>0, 
\end{align} 
where $\eta$ is the regression function as in \eqref{regression function}. Note the assumption becomes trivial for $\alpha=0$ and stronger for larger $\alpha$. An equivalent way to say this is that $|2\eta-1| \in L_{\alpha,\infty}(P_X)$ where $L_{\alpha, \infty}(P_X)$ is the Lorentz space with respect to measure $P_X$. See \cite{tsybakov2004optimal},\cite{audibert2007fast} for further discussion on this assumption.

This assumption is also called the Mammen-Tsybakov noise condition and has the following equivalent characterizations as summarized in \cite{boucheron2005theory}:
\begin{itemize}
\item $\exists \beta>0 \text{ such that for any measurable classifier } g, \text{ we have }\\ \mathbb{E}\left[\mathbbm{1}_{\{x: g(x) \neq M^{*}(x)\}}(X)\right] \leq \beta\left(L(g)-L^{*}\right)^{\kappa} \label{noise condition}$. 

\item $\exists c>0 $ \text{ such that } $\forall A \in \mathcal{B}(\mathbb{R}^d)$, we have \\
    $\int_{A} d P(x) \leq c\left(\int_{A}|2 \eta(x)-1| d P(x)\right)^{\kappa}$.
    
\item $\exists B>0, \forall t \geq 0,$ we have $\mathbb{P}\{|2 \eta(X)-1| \leq t\} \leq B t^{\frac{\kappa}{1-\kappa}}$. 
\end{itemize}
Note $M^*$ is as defined in \eqref{Bayes classifier}, and $\mathcal{B}(\mathbb{R}^d)$ refers to the Borel sigma-algebra of $\mathbb{R}^d$ which actually coincides with the n-product sigma-algebra $\underbrace{\mathcal{B}(\mathbb{R}) \otimes \cdots \otimes \mathcal{B}(\mathbb{R})}_{\text{n times}}$.

Now that we have discussed some standard assumptions, we present a number of important known convergence rates from the literature.

\subsection{Results on ERM classifiers}\label{subsection: Results on ERM classifiers}
Let $\mathcal{C}$ be a given class of sets in $\mathbb{R}^d$. ERM classifier is defined as the function
\begin{align}
    M_{\widehat{G}}(X) \defeq \mathbbm{1}_{\{X: X\in \widehat{G}\}}(X) \defeq
    \begin{cases}
    1, & \text{if $X\in \widehat{G}$};\\
    0, & \text{otherwise.}
    \end{cases}
\end{align}
where 
\begin{align}
    \widehat{G} \defeq \argmin_{G \in \mathcal{C}} \frac{1}{n}\sum_{i=1}^n \mathbbm{1}(\mathbbm{1}_{G}(X_i)\neq Y_i). 
\end{align}
In this framework, Bayes classifier corresponds to $M_{G^*}$ where $G^* = \{x:\eta(x)\geq 1/2\}$.\\ 
As can be seen from the definition, ERM classifiers are completely determined by the choice of decision set $\widehat{G} \subset \mathbb{R}^d$. 
A frequently used assumption on the distribution is that $G^*$ belongs to a certain class of sets, say $\Sigma$, whose complexity is bounded. 
Excess risks of ERM classifiers have been extensively studied under various assumptions on $\mathcal{C}$ and the underlying distribution $P$, and convergence results take the form
\begin{align} \label{eq16}
    \sup E[L(M_{\widehat{G}}(\mathbf{X}))] - L(M^*) = O(n^{-\beta})
\end{align}
for some $\beta>0$. Here the supremum is taken over the class of distributions such that $G^* \in \Sigma$ and satisfy additional assumptions such as regularity conditions on the marginal distribution of $X$. Many of the proofs in this direction rely on results from empirical process theory. 

First, \cite{mammen1999smooth} showed that \eqref{eq16} holds with $\alpha$ that depends on the margin assumption (see \eqref{margin assumption}) as well as the complexity of the true candidate family of sets that $G^*$ belongs to (see Section \ref{dist assumption 1}). The main result therein is that for all $P$ satisfying the margin assumption and complexity assumption, there exists ERM-type classifier $M_n$ such that
\begin{align}
    E_n[L(M_n)] - L^* = O(n^{-\frac{1+\alpha}{2+\alpha+\alpha\rho}}) \label{erm result1}
\end{align}
where $L(M_n)$ is as in \eqref{loss2} and $E_n$ denotes expectation with respect to samples $\mathcal{D}_n$.

The limitation in \cite{mammen1999smooth} was that the solution to ERM that achieves optimal rate was either infeasible or unrealistic when the class $\mathcal{C}$ that contains $G^*$ is large. Specifically, the empirical minimizer was chosen among the entire candidate family of sets or as a sieve estimator where the sieve is constructed based on prior knowledge of the noise parameter. 

Paper \cite{tsybakov2004optimal} addressed such issues by showing that the same optimal rate as \eqref{erm result1} can be achieved with empirical minimizer even when $\mathcal{C}$ is a finite sieve or $\epsilon$-net over the class of sets containing $G^*$ constructed without prior knowledge of noise condition. Here, the construction of sieve or $\epsilon$-net only assumes knowledge of $\delta$-entropy with bracketing of $\mathcal{C}$.  

\subsection{Results on plug-in estimators}\label{subsection: Results on plug-in estimators}

When we have a function $f$ approximating $\eta$ in some sense (for e.g., in $L^p$), recall $p_f$ from \eqref{eq6}, the corresponding plug-in classifier. Under (A1) Mammen-Tsybakov noise condition with noise parameter $\alpha$, (A2) regularity assumption on the distribution of $\mathbf{X}$, and (A3) smoothness assumption on $\eta$ (such as continuous differentiability or H\"{o}lder condition) parametrized by smoothness index $\beta>0$, \cite{audibert2007fast} showed that fast and even super-fast rates (better than $O(n^{-1})$) are possible. We note that assumption (A3) often implies complexity assumption \eqref{car assumption} since most smooth classes of functions satisfy \eqref{car assumption}; see Section V of \cite{yang1999minimax}.

Specifically, the main result is that there exists a plug-in classifier $M_n$ of the form \eqref{eq6} such that uniformly over all $P$ satisfying the above three assumptions, the following holds:
\begin{align}\label{plug-in result}
     E[P(M_n(\mathbf{X})\neq Y)] - L^* = O(n^{-\frac{\beta(1+\alpha)}{2\beta + d}})
\end{align}
where $d$ is the dimension of the input space.

This result and the result of the ERM classifier differ critically in the assumptions they make: while ERM makes an assumption on the complexity of the class of sets, here the assumption is on the complexity of function class $\eta$ belongs to. As noted in \cite{audibert2007fast}, no inclusion relationship holds between these two assumptions. Hence, it is incorrect to nominally compare the rates of \eqref{erm result1} and \eqref{plug-in result}. In \cite{audibert2007fast}, the local polynomial estimator and sieve estimator of $\eta$ are shown to yield minimax optimal rates of convergence when $\eta$ is assumed to belong to H\"{o}lder class of functions.\\
Finally, the work \cite{steinwart2007fast} introduces a new assumption on distributions, namely the geometric noise assumption, which roughly speaking, controls the measure $|2\eta(x)-1|P_X$ near the decision boundary. The noise level is parametrized by $\beta$ in an analogous way $\alpha$ controls noise in the margin assumption. They show that support vector machines based on Gaussian RBF kernels achieve the rates $O(\frac{\beta}{2\beta+1})$ and $O(n^{-\frac{2\beta(\alpha+1)}{2\beta(\alpha+2) + 3\alpha + 4}})$ in two different regimes according to whether $\beta \leq \frac{\alpha+2}{2\alpha}$. One distinguishing feature of this work is that the geometric noise assumption makes no smoothness assumption on $\eta$ or regularity condition on $P_X$ as in the works discussed above. We remark that the rates shown here are implicitly dependent on the input-dimension by the way geometric noise assumption is defined. The analysis in this work is involved mostly because of the explicit regularization term in the loss function of SVMs. One caveat is that no results are known regarding the optimality of these rates in this regime.

\section{Main Results}\label{section:main results}

We assume the basic setup from Section \ref{basic setup}. Notably, we let $\mathcal{D}_n$ denote the sample of $n$ data points and $P$ the corresponding distribution. Furthermore, we suppose that $\mathcal{F}_n$ is a class of fully-connected feed-forward ReLU neural networks. We state the details on what $\mathcal{F}_n$ we consider in the following assumption:
\begin{assumption}[Assumptions on Estimating Function Class]\label{assumption: Assumptions on Estimating Function Class}
    We restrict ourselves to the function class $\mathcal{F}_n \defeq \bigcup_{\mathbf{p}}\mathcal{F}(L,p)$ whose depth $L$ is a fixed constant greater than $10$ and width vector $p$ is such that the total number of parameters is bounded by some function of $n$ denoted by $W(n)$, and such that the range of functions in the class are contained in the interval $[M/2,M/2]$ for some large enough $B>0$.
\end{assumption}
 For the logistic loss $\phi$ (cf. \eqref{logistic loss}), suppose that $\widehat{f}_n$ is the empirical $\phi$-risk minimizer   
\begin{align}\label{empirical phi risk minimizer}
    \widehat{f}_n \defeq \argmin_{f \in \mathcal{F}_n} \frac{1}{n} \sum_{i=1}^n \phi\bullet f(X_i,Y_i).
\end{align}
Recall the $\bullet$ operator defined in \eqref{bullet function}.
The problem we investigate is the uniform rate at which the excess risk $\mathcal{E}(\widehat{f}_n)$ converges to $0$ over the class of distributions, i.e., the set of probability measures $P$, satisfying the following assumptions:
\begin{assumption}[Assumptions on Distribution]\label{assumptions on distribution}
We restrict ourselves to the Borel distributions of $(X,Y)$ satisfying all of the below:
\begin{itemize}
\item The regression function $\eta$ defined in \eqref{regression function} satisfies the margin assumption \eqref{margin assumption}.
    
\item The regression function $\eta$ is bounded away from $0$ and $1$ by some arbitrary constant almost surely. 

\item X is supported on a compact subset $\Omega$ of $\mathbb{R}^d$.
        
\item For some positive integer $M$, there exists an open cover $\{U_i \}_{i=1,\dots,M}$ of $\Omega$, such that over each set $U_i$, there exists a 1 hidden-layer neural network $I_i: \mathbb{R}^d \rightarrow \mathbb{R}$ whose restriction to $U_i$ satisfies all the requirements of Definition \ref{barron approximation space} so as to make $\eta|_{U_i} \in \mathcal{BA}(U_i)$ (cf. Section \ref{subsection: Barron approximation space}). 

\end{itemize}
\end{assumption}

The margin assumption ensures that the regression function remains smooth near the decision boundary where the degree of smoothness is governed by $\alpha$. Bigger $\alpha$ ensures a more favorable situation where there are no jumps at the decision boundary. Moreover, the assumption that $\eta$ is bounded away from $0$ and $1$ can be relaxed by requiring exponential decay of probability of $x$ such that $\eta(x)$ is near $0$ and $1$, but we work with the current assumption for simplicity.

The last assumption roughly says that $\eta$ is locally characterized as a function in the Barron approximation space. Note the choice of open cover may be an infinite one but can be reduced to a finite cover by the compactness of $\Omega$. Similarly, we may assume that each $U_i$ is bounded. In the paper \cite{caragea2020neural} where Barron approximation space is defined and analyzed, they define the so-called sets with Barron class boundary, which says that $\Omega$ is locally defined as the graph of a function in the Barron approximation space: that is, $\mathbbm{1}_{\Omega \cap Q_i}(x) = \mathbbm{1}_{\{ x:x_j \leq f(x^{(j)})\}}(x) $ for some $j\in\{1,\dots,d\}$ where $x^{(j)}$ denotes the $d-1$-dimensional vector formed from $x$ by dropping its $j$th component. However, their $\Omega$ is less general than our set $\{x: \eta(x)\geq 1/2\}$ because they assume the relationship $Y= \mathbbm{1}_{\Omega}(X)$ so that actually $Y$ is $\sigma(X)$-measurable. Thus, our setting of describing the decision set includes the description of $\Omega$ in \cite{caragea2020neural} as a special case. 

\subsection{Approximate excess risk bound}\label{subsection: 4.1}

The first preliminary result shows the rate at which the approximate excess $\phi$-risk defined in \eqref{approximate excess phi risk} converges to $0$. Here, there is no assumption on the space of functions to which $\eta$ belongs. 
\begin{theorem}\label{theo1}
    We assume that the following holds for all integers $n\geq1$: First, suppose that $\mathcal{F}_n$ satisfies Assumption \ref{assumption: Assumptions on Estimating Function Class}.
    Second, suppose that $\Tilde{f} \in \mathcal{F}_n$ satisfies $\Tilde{f} = \argmin_{f\in \mathcal{F}_n}P(\phi\bullet f)$. Third, let $\{\tau_n\}$ be a sequence of positive numbers such that for distribution $P$ there exists a neural network $I_n \in \mathcal{F}_n$ satisfying $\mathcal{E}_{\phi}(I_n) \leq c_0\tau_n $ for some constant $c_0>0$. Denote
    \begin{align*}
        \omega_n(\delta) &\defeq E \sup_{f\in \mathcal{F}_n, \norm{f-\Tilde{f}}_{L_2(P_X)}^2 \leq \delta} |R_n(f-\Tilde{f})|,\\
        \widehat{f}_n &\defeq \argmin_{f\in\mathcal{F}_n}E_n(\phi \bullet f),
    \end{align*}
    where $E_n$ denotes expectation with respect to empirical measure based on $\mathcal{D}_n$
    Then, there exists constants $K>0, C>0, c>0$ such that for all $\alpha \in (0,1]$,
    \begin{align*}
        &P\left(\widehat{\mathcal{E}}_{\phi}(\widehat{f}_n) \geq K\left(\max\left\{\omega_n^{\sharp}\left(c\alpha\right)-\tau_n, \tau_n\alpha\right\} +\frac{t}{n} + \sqrt{\frac{t\tau_n}{n}} \right)\right) \\
        &\leq Ce^{-t}
    \end{align*}
    where $\omega_n^{\sharp}$ refers to the $\sharp$-transformation \eqref{transformations} of the function $\omega_n$.  
    In particular, if there exists a neural network $I_n \in \mathcal{F}_n$ such that $\mathcal{E}_{\phi}(I_n) \leq \frac{C}{\sqrt{W(n)}}$ for some constant $C>0$, for all $\alpha \in (0,1]$, we have
    \begin{align}\label{pre excess bound}
        &P\biggl(\widehat{\mathcal{E}}_{\phi}(\widehat{f}_n) \geq K\biggl(\max\left\{\omega_n^{\sharp}\left(c\alpha\right)-\tau_n, \tau_n\alpha\right\} +\frac{t}{n} \nonumber\\
        & \qquad + \sqrt{\frac{t}{n\sqrt{W(n)}}} \biggr)\biggr) \leq Ce^{-t}.
    \end{align}
\end{theorem}

\subsection{Approximation of the Barron Approximation Space}\label{subsection: 4.2}

In this section, we give some intermediate approximation results regarding the regression function $\eta$. We begin by summarizing the direct approximation properties of $\eta$ by neural networks in the $L^{\infty}$-norm. 

\begin{lemma}\label{lemma: eta approximation}
    Suppose the regression function $\eta$ satisfies Assumption \ref{assumptions on distribution}. Then, there exists a neural network $\Tilde{f} \in \mathcal{F}(11,(d,p,\dots,p,1))$ such that $\norm{\eta-\Tilde{f}}_{\infty}\leq Cp^{-\frac{1}{2}}$ for constant $C>0$ only dependent on $d$.
\end{lemma}

While \cite{caragea2020neural} uses the so-called tube-compatible assumption on the distribution of $X$ to obtain a similar result, we relax this assumption by making use of the Mammen-Tsybakov noise condition. Note that the surrogate loss function $\phi$ is classification-calibrated and from the Definition \ref{classification calibrated}, we observe the minimizer $f_{\phi}^*(x) = H(\eta(x))$ satisfies   
\begin{align*}
    \mathbbm{1}_{f_{\phi}^*(x)\geq 1/2} = \mathbbm{1}_{\eta(x)\geq 1/2}.
\end{align*}
In fact, since our $\phi$ is the logistic loss, there is a closed-form expression for $f_{\phi}^*$ as
\begin{align}\label{phi_risk_minimizer}
    f_{\phi}^*(x) = \log\left(\frac{\eta(x)}{1-\eta(x)} \right) .
\end{align}
Now observe that $\log(x/(1-x))$ for $0<x<1$ is smooth and furthermore Lipschitz continuous when $x$ is restricted to a compact subset of $(0,1)$. Hence, it is not hard to see that $f_{\phi}^*  \in \mathcal{BA}(Q_m)$.
We make this statement precise in the lemma below:
\begin{lemma}\label{lemma_approx}
    Suppose the regression function $\eta$ satisfies the second and the last bullet point of Assumption \ref{assumptions on distribution}. Then, if we define $f_{\phi}^*(x) = \log\left(\frac{\eta(x)}{1-\eta(x)} \right)$, there exists $I_1 \in \mathcal{F}(2,\overline{p})$ for $\overline{p}=(p,p)\in \mathbb{N}^2$ such that $\norm{f_{\phi}^*|_{Q_m} - I_1}_{Q_m, \infty} \leq C p^{-\frac{1}{2}}$. Moreover, for any $\delta>0$, there exists $I_2 \in \mathcal{F}(3,\Tilde{p})$ for $\Tilde{p}=(2d+p,2d+p,2d+1, 1)\in \mathbb{N}^4$ and some $\Omega_0 \subset \Omega$ with probability at least $1-\delta$ such that  $\norm{f_{\phi}^* - I_2}_{\Omega_0, \infty} \leq Cp^{-\frac{1}{2}}$ for constants $C>0$ only dependent on $d$.
\end{lemma}
We defer the proof to Section \ref{section: proofs}.

Now, we provide an approximation result which shows the existence of a sequence of neural network functions that achieves the rate $O(N^{-1/2})$ for the excess $\phi$-risk.

\begin{theorem}\label{approximation theorem}
    Let the input dimension $d$ be an integer greater than $1$, and suppose the regression function $\eta$ satisfies Assumption \ref{assumptions on distribution}. 
    Let $f_{\phi}^*$ be the minimizer of $\phi$-risk as in \eqref{phi_risk_minimizer}. 
    Then, there exists a neural network $I_{N}$ with 3 hidden layers with $O(N)$ parameters such that we have
    \begin{align}\label{approximation bound display}
    \mathcal{E}_{\phi}(I_N) = E(\phi\bullet I_N) - E(\phi \bullet f_{\phi}^*) \leq CN^{-1/2}.
    \end{align}
    where $C$ may depend on $d$ linearly.
\end{theorem}

\subsection{Classification Error Bounds}\label{subsection: 4.3}

In this section, we construct an appropriate class of neural networks and apply Theorem \ref{theo1} to obtain a convergence rate for the excess risk. Note that Theorem \ref{theo1} only guarantees the convergence of the approximate excess $\phi$-risk, $\widehat{\mathcal{E}}_{\phi}(\widehat{f}_n)$. To make use of \eqref{bartlett bound} that relates the excess risk to the excess $\phi$-risk, we observe that with the simple decomposition
$\mathcal{E}_{\phi}(f) = E[\phi \bullet f] - \inf_{g \in \mathcal{F}_n}E[\phi \bullet g] + \inf_{g \in \mathcal{F}_n}E[\phi \bullet g] - E[\phi \bullet f^*_{\phi}]$, it will be made straightforward that the same rate applies to the excess $\phi$-risk by the way we have constructed $\mathcal{F}_n$. In particular, its complexity is controlled so that the bottleneck from approximation and estimation error are the same.


The first step consists of combining Theorem \ref{theo1} with Theorem \ref{approximation theorem} to obtain a rate of convergence for the approximate excess $\phi$-risk when $\mathcal{F}_n$ is an appropriately chosen class of neural network functions. The result is given in the following proposition:

\begin{proposition}\label{approximate excess risk rate}
    Suppose that $\mathcal{F}_n$ is a class of fully-connected feed-forward ReLU neural networks satisfying Assumption \ref{assumption: Assumptions on Estimating Function Class} with $W(n)$ being some constant multiple of $n^{2/3}$. Denote by $\Sigma$ the set of joint distributions on $(X,Y)$ satisfying Assumption \ref{assumptions on distribution}. Let $\widehat{f}_n$ be the empirical $\phi$-risk minimizer: 
    \begin{align*}
        \widehat{f}_n &\defeq \argmin_{f\in\mathcal{F}_n}E_n(\phi \bullet f).
    \end{align*}
    Then, 
    \begin{align}\label{proposition 4.5 display}
    \sup_{P \in \Sigma}P\left(\widehat{\mathcal{E}}_{\phi}(\widehat{g}_n) \geq K \left( 1+t \right)n^{-\frac{1}{3}}\log(n) \right) \leq Ce^{-t}.
\end{align}
\end{proposition}

Now, to apply either Zhang's inequality \eqref{zhang's inequality} or Bartlett's improved bound \eqref{bartlett bound}, we need to show convergence of the excess $\phi$-risk. This type of situation is also noted in \cite{blanchard2003rate}, in which they simply assume $\inf_{f \in\mathcal{F}_{\lambda}}L_{\phi}(f) = L_{\phi}(f^*_{\phi})$ when $\mathcal{F}_{\lambda}$ is a class of regularized (by parameter $\lambda$) boosting classifiers. This is true in simple cases: for example, when $f$ is of bounded variation and the base classifiers in boosting are decision stumps. For us, this condition is not true, but a simple inspection of the rates \eqref{approximation bound display} and \eqref{proposition 4.5 display} shows that indeed the same rate as \eqref{proposition 4.5 display} may be claimed for $\mathcal{E}_{\phi}(\widehat{g}_n)$. Applying \eqref{bartlett bound}, the conclusion is summarized as follows:


\begin{theorem}\label{main excess bound}
    Suppose that $\mathcal{F}_n$ is a class of fully-connected feed-forward ReLU neural networks satisfying Assumption \ref{assumption: Assumptions on Estimating Function Class}. Let $\mathcal{D}_n$ be a given $n$ samples of data: $\mathcal{D}_n = \{(X_1,Y_1), \dots (X_n,Y_n)\}$ each i.i.d. from a common distribution $P$ that belong to a class of distributions $\Sigma$ that satisfy Assumption \ref{assumptions on distribution}. Then, for some constant $C>0$, the empirical $\phi$-risk minimizer defined as
    \begin{align*}
        \widehat{f}_n \defeq \argmin_{f \in \mathcal{F}_n} \frac{1}{n} \sum_{i=1}^n \phi \bullet f(X_i,Y_i),
    \end{align*}
    satisfies the following:
    \begin{align}\label{main convergence rate}
        &\sup_{P\in \Sigma} E[P(p_{\widehat{f}_n}(X) \neq Y)] - P(M^*(X) \neq Y)\nonumber \\
        &\leq C n^{-\frac{1+\alpha}{3(2+\alpha)}}(\log(n))^{\frac{1+\alpha}{2+\alpha}}.
    \end{align}
\end{theorem}
Note that the rate is increasing in $\alpha$ from $n^{-\frac{1}{6}}$ in the worst case to $n^{-\frac{1}{3}}$ as $\alpha \rightarrow \infty$, which indicates a more favorable situation for learning, i.e., near-jumps at the boundary of the regression function.

\subsection{Minimax Lower Bound}\label{subsection: 4.4}
In this subsection, we present a minimax lower bound corresponding to the class of distributions considered in the previous subsection. The result shows that the upper bound of Theorem \ref{main excess bound} is tight up to a logarithmic factor.
\begin{theorem}\label{minimax lower bound theorem}
For the class of distributions $\Sigma$ that satisfy Assumption \ref{assumptions on distribution}, the lower bound for the minimax excess risk is given by
\begin{align*}
    &\inf_{\widehat{f}_n}\sup_{P\in \Sigma}E\left[P(\widehat{f}_n(\mathbf{X})\neq Y) - P(M^*(\mathbf{X})\neq Y) \right]\\
    &\geq Cq^{-r}mw(1-q^{-r}\sqrt{nw})
\end{align*}
for some constant $C>0$, and any positive integer $q, m$ and positive real $w$ satisfying $m\leq q^d$, $w \leq \frac{1}{m}$, and $wm\leq \frac{q^{-r\alpha}}{2^{\alpha}}$. In particular, choosing $m=q^d$, $w=\frac{q^{-\alpha r-d}}{2^\alpha}$, $r=\frac{2d}{2+\alpha}$ for $q=\lfloor \overline{C} n^{\frac{1}{3r(2+\alpha)}}\rfloor$, the lower bound becomes $Cn^{-\frac{1+\alpha}{3(2+\alpha)}}$.
\end{theorem}
We remark that the rate ranges from $n^{-\frac{1}{6}}$ to $n^{-\frac{1}{3}}$ as $\alpha$ varies from $0$ to $\infty$. We note that the regime we study is indeed an inherently difficult one as general as it may be, and that fast rates better than $n^{-\frac{1}{2}}$ are not possible as in \cite{massart2006risk}, \cite{tsybakov2004optimal}, \cite{audibert2007fast}, even with the margin assumption imposed on the class of distributions. As noted in the discussion section of Tsybakov's work \cite{tsybakov2004optimal}, margin assumption has served as a key assumption separating results with rates slower or faster than $n^{-\frac{1}{2}}$. Our work then adds a new part of the picture: how does the margin assumption affect convergence rate when the class of distributions is so large that fast rates are not possible? In fact, a bound on the metric entropy, i.e., the logarithm of the minimal covering number, is sufficient to ensure that for large $\alpha$, fast rates are possible (c.f. \cite{kerkyacharian2014optimal},\cite{tsybakov2004optimal}). Further assumptions on the smoothness of regression function and regularity of the support of $P_X$ can even lead to super-fast rates faster than $n^{-1}$ (c.f. \cite{audibert2007fast}). Theorem \ref{minimax lower bound theorem} then shows that such fast rates are not possible in our set-up and furthermore, that the rate \eqref{main convergence rate} achieved by the empirical $\phi$-risk minimizer is indeed minimax optimal.

\section{Proofs}\label{section: proofs}

\subsection{Proof of Theorem \ref{theo1}}\label{subsection: 5.1}
The first step is to obtain the standard excess risk bound using the techniques from \cite{koltchinskii2011oracle}. Recalling the definitions of expected sup-norm \eqref{expected sup norm} and $L_2(P)$-diameter of $\delta$-minimal set \eqref{diameter}, for any $q>1$ and $t>0$, denote
$$
V_{n}^{t}(\sigma):=2 q\left[\phi_{n}^{b}(\sigma)+\sqrt{\left(D^{2}\right)^{b}(\sigma)} \sqrt{\frac{t}{n \sigma}}+\frac{t}{n \sigma}\right], \sigma>0.
$$
Let
$$
\sigma_{n}^{t}:=\sigma_{n}^{t}(\mathcal{F} ; P):=\inf \left\{\sigma: V_{n}^{t}(\sigma) \leq 1\right\}.
$$
\begin{proposition}\label{excess risk bound}
    Suppose that functions in $\mathcal{F}$ take values in $[0,1]$. Let
    $$
    \widehat{f}_n \defeq \argmin_{f\in\mathcal{F}} P_n(f).
    $$
    Then, For all $t>0$,
    $$
    \mathbb{P}\left\{\widehat{\mathcal{E}}\left(\widehat{f}_{n}\right)>\sigma_{n}^{t}\right\} \leq C_{q} e^{-t}
    $$
    where
    $$
    C_{q}:=\max \left\{ \frac{q}{q-1}, e \right\}.
    $$
\end{proposition}
In fact, the above proposition easily extends to function classes with values in a compact subset of $\mathbb{R}$ with only changes to constants.\\
The second step is to obtain a bound on the second moment of the process $\{\phi\bullet g\}_{g \in \mathcal{G}_n}$ in terms of its first moment.

For a vector space $S$, define modulus of convexity of a convex function $f:S \rightarrow \mathbb{R}$ with respect to metric $d$ as 
\begin{align*}
    &\delta(\epsilon) \\
    &\enspace \defeq \inf\biggl\{\frac{f(x)+f(y)}{2} - f\left(\frac{x+y}{2}\right): x,y \in S, d(x,y)\geq \epsilon \biggr\}
\end{align*}
and call $f$ strictly convex with respect to $d$ if $\delta(\epsilon)>0$ for all $\epsilon>0$.
The first lemma shows the logistic loss satisfies some convexity conditions and is essentially proved in \cite{bartlett2006convexity}:

\begin{lemma}\label{modulus_convexity}
    Suppose random variable $X$ is compactly supported, and $\mathcal{F}$ is a class of functions defined on the support of $X$ such that for all $f\in\mathcal{F}$, $f(X)$ is almost surely contained in $[-M/2, M/2]$ for some constant $M>0$. For logistic loss $\phi$, $L_{\phi}(f)=E[\phi((2Y-1)f(X))]$ is a strictly convex functional of $f$ with respect to $L^2$ distance $d$, which is defined for two integrable functions $f,g$ as $d(f,g) = E[(f(X)-g(X))^2]^{\frac{1}{2}}$, and the modulus of convexity of this functional $L_{\phi}$ satisfies $\delta(\epsilon) \geq \frac{e^{-M}}{16}\epsilon^2$.
\end{lemma}

\begin{proof}
    Using the fact that $\phi$ restricted to $[-M/2,M/2]$ has modulus of convexity $e^{-M}\epsilon^2/16$ (see for e.g., Table 1 of \cite{bartlett2006convexity}), Lemma 8 of \cite{bartlett2006convexity} immediately gives the result.
\end{proof}

We write $\delta_{L_{\phi}}(\cdot)$ to denote the modulus of convexity of the functional $L_{\phi}$.
Now we present our desired second moment bound:

\begin{lemma}\label{variance_bound}
    Let $\mathcal{F}_n$, $I_n$, and $P$ be as in Theorem \ref{theo1}. Let $M>0$ be such that the range of all functions in $\mathcal{F}_n$ is contained in $[-\frac{M}{2}, \frac{M}{2}]$. Suppose that $\Tilde{f} \in \mathcal{F}_{n}$ satisfies $\Tilde{f} = \argmin_{f\in \mathcal{G}_{n}}E
    (\phi\bullet f)$ 
    Then, for any $f \in \mathcal{F}_n$, we have
    $$ E(\phi\bullet f - \phi\bullet\Tilde{f})^2\leq E(f-\Tilde{f})^2 \leq 8e^{M}E(\phi\bullet f - \phi \bullet \Tilde{f}) + 2c_0\tau_n.$$
\end{lemma}

\begin{proof}
    First, we have
    \begin{align*}
        &E[(\phi\bullet f - \phi\bullet\Tilde{f})^2]\\
        &=E[\phi((2Y-1)f(X)) - \phi((2Y-1)\Tilde{f}(X))^2]\\
        &\leq E[(f(X)-\Tilde{f}(X))^2].\\
    \end{align*}
    Second, using the notation
    \begin{align*}
        L_{\phi}(f) \defeq E[(2Y-1)\phi(f(X))],
    \end{align*}
    we can argue by Lemma \ref{modulus_convexity} that 
    \begin{align*}
        \frac{L_{\phi}(f) + L_{\phi}(\Tilde{f})}{2} &\geq L_{\phi}\left(\frac{f+\Tilde{f}}{2}\right) + \delta_{L_{\phi}}\left(E[(f(X)-\Tilde{f}(X))^2]^{\frac{1}{2}}\right)\nonumber\\
        &\geq L_{\phi}\left(\frac{f+\Tilde{f}}{2}\right) + \frac{e^{-M}}{16}E[(f(X)-\Tilde{f}(X))^2].
    \end{align*}
    Now, observe that $\frac{f + \Tilde{f}}{2}$ belongs to a neural network class $\mathcal{F}'$ with at most twice as many parameters as $\mathcal{G}_n$ (see, for example \cite[Lemma 2.6]{elbrachter2021deep}),
    \begin{align*}
        L_{\phi}(\Tilde{f}) - L_{\phi}\left( \frac{f+\Tilde{f}}{2} \right)  &\leq L_{\phi}(\Tilde{f}) - \inf_{g \in \mathcal{F}'} L_{\phi}(g)\\
        &\leq \underbrace{L_{\phi}(\Tilde{f}) - L_{\phi}(I_n)}_{\leq 0} + L_{\phi}(I_n) - L_{\phi}(f^*_{\phi}) \\
        &+ \underbrace{L_{\phi}(f^*_{\phi}) - \inf_{g \in \mathcal{G}'} L_{\phi}(g)}_{\leq 0} \\
        &\leq c_0\tau_n
    \end{align*}
    where the last inequality follows by assumption. Thus, we have
    \begin{align}
        \frac{L_{\phi}(f) - L_{\phi}(\Tilde{f})}{2} &\geq  \frac{e^{-B}}{16}E[(f(X)-\Tilde{f}(X))^2] - c_0\tau_n, \nonumber
    \end{align}
    so that we can conclude
    \begin{align}
        E[\phi\bullet f - \phi\bullet\Tilde{f}] + 2c_0\tau_n  &\geq \frac{e^{-2B}}{8}E[(f(X)- \Tilde{f}(X))^2]\nonumber \\
        &\geq \frac{e^{-2B}}{8}E[(\phi\bullet f - \phi\bullet\Tilde{f})^2].\nonumber
    \end{align}
\end{proof}

There is one more technical lemma regarding $\sharp$-transformation stated in \cite{koltchinskii2006local} we will need:
\begin{lemma}\label{sharptransformation_property} 
    For any given function $\psi:\mathbb{R}_{\geq0} \rightarrow \mathbb{R}_{\geq 0}$ where $\mathbb{R}_{\geq0}$ denotes the set of non-negative real numbers, the following holds:
\begin{enumerate}
\item For $c>0$, let $\psi_{c}(\delta)\defeq\psi(c \delta)$. Then $\psi_{c}^{\sharp}(\epsilon)=\frac{1}{c} \psi^{\sharp}(\epsilon / c)$.
        
\item For $c>0$, let $\psi_{c}(\delta)\defeq \psi(\delta+c)$. Then for all $u>0, \epsilon \in(0,1], \psi_{c}^{\sharp}(u) \leq$ ($\psi^{\sharp}(\epsilon u / 2)-c) \vee c \epsilon .$ where $a \vee b = \max\{a, b\}$.

\item For $\epsilon = \epsilon_1 + \cdots + \epsilon_m$ and $\psi_1,\dots,\psi_m$ as in the hypothesis, $(\psi_1 + \cdots + \psi_m)(\epsilon) \leq \psi_1(\epsilon_1) + \cdots + \psi(\epsilon_m)$.
\end{enumerate}
\end{lemma}

Now we finally prove our first theorem:
\begin{proof}[Proof of Theorem \ref{theo1}]
    From Lemma \ref{variance_bound}, we have that for any $f\in \mathcal{F}_n$,
    \begin{align*}
        \frac{e^{-M}}{8}P_X(f(X)-\Tilde{f}(X))^2 -2c_0\tau_n \leq E[\phi\bullet f - \phi\bullet\Tilde{f}],
    \end{align*}
    which implies that the $\delta$-minimal set of $\mathcal{L} = \{\phi\bullet f: f\in \mathcal{F}_n\}$ satisfies
    \begin{align*}
        \mathcal{L}(\delta) &= \{\phi \bullet f: f \in \mathcal{F}_n, \widehat{\mathcal{E}}_{\phi}(f, \mathcal{F}_{n}) \leq \delta \}\\
        &\subset \{\phi \bullet f: P_X(f(X) - \Tilde{f}(X))^2 \leq 16e^{M}\left(\delta + c_0\tau_n\right) \}.
    \end{align*}
     Letting $\mathcal{F}_{\delta} \defeq \{f: f \in \mathcal{F}_n, P_X(f(X) - \Tilde{f}(X))^2 \leq 16e^{M}\left(\delta + c_0\tau_n\right) \}$, we can thus bound the $L_2(P)$ diameter of $\mathcal{L}(\delta)$ as
    \begin{align*}
        D^2(\delta) &\defeq D^2(\mathcal{L},\delta)\\
        &= \sup_{f_1, f_2 \in \mathcal{L}(\delta)}P(\phi \bullet f_1 - \phi \bullet f_2)^2\\
        &\leq \sup_{f_1,f_2 \in \mathcal{F}_{\delta}} P_X(f_1 - f_2)^2\\
        &\leq \sup_{f_1,f_2 \in \mathcal{F}_{\delta}} 2P_X(f_1-\Tilde{f})^2 +  2P_X(f_2-\Tilde{f})^2\\
        &\leq 64e^{M}\left(\delta+c_0\tau_n\right)
    \end{align*}
    where first inequality used Lipschitz property of $\phi$, second inequality used the elementary inequality $(a+b)^2 \leq 2a^2 + 2b^2$, and third inequality follows from the definition of $\mathcal{G}_{\delta}$.
    Secondly, we seek to bound $\phi_n(\delta)$. Using symmetrization, properties of Rademacher complexity, and the contraction principle, we can write
    \begin{align*}
        &\phi_n(\mathcal{L},\delta) \\
        &= E\left[ \sup_{g_1,g_2 \in \mathcal{L}(\delta)} |(P_n- P)(g_1-g_2)| \right]\\
        &\leq 2E\left[\sup \left\{\left| R_n(\phi \bullet f_1 - \phi \bullet f_2) \right| : \phi \bullet f_1, \phi \bullet f_2\in \mathcal{L}(\delta) \right\} \right]\\
        &\leq 2E\left[\sup \left\{\left| R_n(\phi \bullet f_1 - \phi \bullet f_2) \right| : f_1, f_2\in \mathcal{F}_{\delta} \right\} \right]\\
        &\leq 4E\left[\sup_{f \in \mathcal{F}_{\delta}} \left|\frac{1}{n}\sum_{i=1}^n \epsilon_i\left(\phi\bullet f(X_i,Y_i) - \phi\bullet \Tilde{f}(X_i,Y_i) \right) \right|\right]\\
        &\leq 8E\left[\sup_{g \in \mathcal{F}_{\delta}} \left|\frac{1}{n}\sum_{i=1}^n \epsilon_i\left(f(X_i) - \Tilde{f}(X_i) \right) \right|\right]\\
        &= 8\omega_n(8e^{M}(\delta+\tau_n)).
    \end{align*}
    From the above calculations, we can bound
    \begin{align*}
        V_n^t(\delta) &\defeq V_n^t(\delta, t)\\
        &=4\left(\phi_n^b(\mathcal{L},\delta) + \sqrt{(D^2)^b(\delta)}\sqrt{\frac{t}{n\delta}} + \frac{t}{n\delta}\right)\\
        &\leq K_1 \biggl( \sup_{\sigma\geq \delta}\left\{\frac{\omega_n(8e^{M}(\sigma+\tau_n))}{\sigma}\right\} \\
        &+  \sqrt{\sup_{\sigma\geq\delta}\frac{\sigma+\tau_n}{\sigma}}\sqrt{\frac{t}{n\delta}}  + \frac{t}{n\delta} \biggr)\\
        &\leq K_2 \biggl( \sup_{\sigma\geq \delta}\left\{\frac{\omega_n(8e^{M}(\sigma+\tau_n))}{\sigma} \right\} + \sqrt{\frac{t}{n\delta}} + \sqrt{\frac{t\tau_n}{n\delta^2}}\\
        &+ \frac{t}{n\delta} \biggr),
    \end{align*}
    for some $K_1, K_2 \geq 2$.
    Further using the three properties of $\sharp$-transformation from Lemma \ref{sharptransformation_property}, we can conclude that for some constant $c>0$ only dependent on $M$,
    \begin{align*}
        \sigma_{n}^{t} &\defeq \inf \left\{\sigma>0: V_{n}^{t}(\sigma) \leq 1\right\}\\
        &\leq K_3\left(\max\left\{\omega_n^{\sharp}\left(c\alpha\right)-\tau_n, \tau_n\alpha\right\} +\frac{t}{n} + \sqrt{\frac{t\tau_n}{n}}\right)
    \end{align*}
    where $\alpha$ is any number in $(0,1]$, and $K_3$ and $c$ depend only on $M$. Now, we can immediately apply Proposition \ref{excess risk bound} to obtain the desired result.
\end{proof}

\subsection{Proof of Lemma \ref{lemma: eta approximation}}
We first state a result on approximating multiplication operator with a neural network adapted from \cite{petersen2018optimal} for our purpose:

\begin{lemma}[\cite{petersen2018optimal}, Lemma A.3]\label{lemma: multiplication}
    Let an arbitrary real number $M>0$ be fixed. Then, there exists a neural network $Mul(\cdot) \in \mathcal{F}(9, (2,p,\dots,p,1))$ such that for all $x,y \in [-M,M]$, we have
    \begin{align*}
        |xy - Mul(x,y)| \leq Cp^{-\frac{1}{2}}
    \end{align*}
    for $C>0$ depending only on $M$. Moreover, for any $x,y$ such that $xy=0$, we also have $Mul(x,y)=0$.
\end{lemma}

\begin{proof}[Proof of Lemma \ref{lemma: eta approximation}]
    By assumption, there exists an open cover of $\Omega$, $\{ U_i \}_{i=1,\dots,M}$ and corresponding neural networks $\{I_i\}_{i=1,\dots,M}$, each in $\mathcal{F}(1,p)$, such that $\norm{\mathbbm{1}_{U_i}(\eta-I_i)}_{\infty} \leq Cp^{-\frac{1}{2}}$ for all $i$. Also, there exists a $C^{\infty}$ partition of unity $\{\rho_i\}_{i=1,\dots,M}$ subordinate to the given cover  (c.f. \cite{tu2011manifolds} Theorem 13.7). Because smooth functions belong to the Barron approximation space, there exists 1 hidden-layer neural networks $\Tilde{\rho_i} \in \mathcal{F}(1,p)$ such that $\norm{\rho_i-\Tilde{\rho_i}}_{\infty} \leq C_1p^{-\frac{1}{2}}$. Then, it is easy to see that $\sum_{i=1}^M Mul(I_i,\Tilde{\rho_i})$ is again a neural network with $11$ hidden layers with constant width, which is a constant (depending only on $M$) multiple of $p$. Then, using the decomposition $\eta = \sum_{i=1}^{M}\rho_i \eta$, we have
    \begin{align*}
        \norm{\eta - \sum_{i=1}^M Mul(I_i,\Tilde{\rho_i})}_{\infty} &= \norm{\sum_{i=1}^{M}\rho_i \eta - \sum_{i=1}^M Mul(I_i,\Tilde{\rho_i})}_{\infty}
    \end{align*}
    \begin{align*}
        &\leq \sum_{i=1}^{M} \norm{\rho_i\eta - Mul(I_i, \Tilde{\rho_i})}_{\infty}\\
        &\leq \sum_{i=1}^{M} \norm{\rho_i\eta - Mul(\rho_i, \eta)}_{\infty} + \norm{Mul(\rho_i, \eta) - Mul(\rho_i, I_i)}_{\infty} \\
        &\qquad + \norm{Mul(\rho_i, I_i) - Mul(I_i, \Tilde{\rho_i})}_{\infty}\\
        &\leq Cp^{-\frac{1}{2}}.
    \end{align*}
    Here, note that the network $I_i$'s all have bounded weights (bound only depending on the diameter of $\Omega$), hence it is possible to apply Lemma \ref{lemma: multiplication} with appropriate choice of $M>0$ in the hypothesis.
\end{proof}

\subsection{Proof of Lemma \ref{lemma_approx}}\label{subsection: 5.2}

\begin{proof}
    For a function $f: A \rightarrow \mathbb{R}$ for $A \subset \mathbb{R}^d$ and $B \subset A$, denote by $\norm{f}_{B,\infty}$ the sup-norm of $f$ over $B$, i.e., $\sup_{x\in B}|f(x)|$ Fix any $m\in\{1,\dots,M\}$. By assumption, for any $p\in \mathbb{N}$ there exists a 1 hidden-layer ReLU neural network $I_m^p$ with $p$ neurons  such that 
    \begin{align}\label{eq33}
        \norm{\eta - I_m^p}_{Q_m,\infty} \leq B_m\sqrt{d}p^{-1/2}.
    \end{align}
    Let $c$ denote the constant such that $\eta(x) \in [c, 1-c]$ for all $x$, which exists by assumption. Now let

    \begin{align*}\label{eq34}
        g(x) \defeq
        \begin{cases}
            \log\frac{x}{1-x}, & \text{if }  x\in[c,1-c];\\
            \log\frac{c}{1-c}, & \text{if } x \in [c-B_m\sqrt{d}p^{-\frac{1}{2}}, c);\\
            \log\frac{1-c}{c}, & \text{if } x \in (1-c,1-c+B_m\sqrt{d}p^{-\frac{1}{2}}].
        \end{cases} 
    \end{align*}
    Then, being piecewise infinitely differentiable, there is a neural network $I_0^p$ with $p$ neurons such that
    \begin{align}
        \norm{g - I_0^p}_{[c-B_m\sqrt{d}p^{-\frac{1}{2}},1-c +B_m\sqrt{d}p^{-\frac{1}{2}}],\infty} \leq C_0\sqrt{d}p^{-1/2}
    \end{align}
    for some constant $C_0>0$.
    Now, combining \eqref{eq33} and \eqref{eq34} and using $\circ$ to denote composition of functions, we can write using a simple telescoping argument,
    \begin{align}
        &\norm{f_{\phi}^*|_{Q_m} - I_0^p \circ I_m^p}_{Q_m, \infty} \\
        &= \norm{g\circ \eta|_{Q_m} -I_0^p \circ I_m^p}_{Q_m, \infty} \nonumber\\
        &\leq \norm{g\circ \eta|_{Q_m} - g\circ I_m^p + g\circ I_m^p - I_0^p \circ I_m^p }_{Q_m, \infty}\nonumber\\
        &= \norm{g\circ \eta|_{Q_m} - g\circ I_m^p}_{Q_m, \infty} + \norm{g\circ I_m^p - I_0^p \circ I_m^p}_{Q_m, \infty}\nonumber\\
        &\leq C\norm{\eta|_{Q_m} - I_m^p}_{Q_m, \infty} + \norm{g - I_0^p}_{[c,1-c], \infty} \nonumber\\
        &\leq (CB_m + C_0)\sqrt{d}p^{-1/2} \label{line40}
    \end{align}
    where $C$ only depends on $c$. Note that the composition $I_0^p \circ I_m^p$ can be realized as a neural network as well using the fact that the positive and negative parts of the outputs of $I_m^p$ can be separately treated to get identity. Since $I_0^p \circ I_m^p$ is continuous on a compact set, its outputs are bounded by a constant, say $C_1>0$. Without loss of generality, we may assume $C_1$ is the bound for the outputs of $I_0^p \circ I_m^p$ for all $m=1,\dots, M$.
    
    Next,  we want to "glue together" the neural networks from above in such a way that the final neural network is locally identical to $I_0^p \circ I_m^p$ on each of the rectangles $Q_m$. Fix $m$. Since the measure $P_X$ is a Borel probability measure, it is finite so that it is a regular measure (see, for e.g., Theorem 7.1.4 of \cite{dudley2018real}). This means that for any $A$ in the Borel $\sigma$-algebra $\mathcal{B}(\mathbb{R}^d)$,
    \begin{align*}
        \mu(A) &= \sup \{\mu(K): K \subset A, K \text { compact}, K \subset \mathbb{R}^d \}.
    \end{align*}
    In particular, when $Q_m = [a,b] \subset \mathbb{R}^d$, there exists a real number $\epsilon$ such that $0<\epsilon \leq \frac{1}{2} \min_{i\in[d]}(b_i-a_i)$ and $[a+\epsilon,b-\epsilon] \defeq \prod_{i=1}^d [a_i + \epsilon, b_i-\epsilon]$ satisfies $P_X([a,b]\backslash [a+\epsilon,b-\epsilon]) \leq \delta/M$. Then similar to the approach in \cite{petersen2018optimal} Lemma A.6, we define $t_{i,m}:\mathbb{R} \rightarrow \mathbb{R}, i\in \{1,\dots,d \}$ as
    $$
    t_{i,m}(u) \defeq \begin{cases}0, & \text { if } u \in \mathbb{R} \backslash\left[a_{i}, b_{i}\right]; \\ 1, & \text { if } u \in\left[a_{i}+\varepsilon, b_{i}-\varepsilon\right]; \\ \frac{u-a_{i}}{\varepsilon}, & \text { if } u \in\left[a_{i}, a_{i}+\varepsilon\right]; \\ \frac{b_{i}-u}{\varepsilon}, & \text { if } u \in\left[b_{i}-\varepsilon, b_{i}\right]\end{cases}
    $$
    and $h_{\epsilon,m}:\mathbb{R}^d \times \mathbb{R} \rightarrow \mathbb{R}$, $h_{\epsilon}(x,y) \defeq C_1\sigma(\sum_{i=1}^dt_i(x_i) + \sigma(y)/C_1 - d)$. Now, we observe that 
    \begin{align*}
        h_{\epsilon}(x,y) \defeq 
        \begin{cases}
            y, & \text{ if } x\in [a+\epsilon, b-\epsilon], y\in (0,C_1);\\
            0, & \text{ if } x\in \mathbb{R}\backslash[a, b]\\
        \end{cases} 
    \end{align*}
    Note that $h_{\epsilon}$ is implementable by a neural network, say $J_{m}$.
    Then, assuming for simplicity that the outputs of $I_0^p\circ I_m^p$ are all positive (which is true for $p$ large enough), we have by construction
    \begin{align}
        &P_X(\{ x\in \mathbb{R}^d: \mathbbm{1}_{Q_m}(x)I_0^p\circ I_m^p(x) \neq J_0(x,I_0^p\circ I_m^p(x)) \}) \nonumber\\
        &\leq P_X([a,b]\backslash [a+\epsilon,b-\epsilon]) \nonumber\\
        &\leq 1-\delta/M. \label{line43}
    \end{align}
    Combining \eqref{line40} and \eqref{line43}, we can conclude that on a set with probability at least $1-\delta$, say $\Omega_0$,
    \begin{align*}
        \norm{f_{\phi}^* - \sum_{m=1}^M J_m(\cdot,I_0^p\circ I_m^p) }_{\Omega_0, \infty} \leq (CB + C_0)\sqrt{d}p^{-1/2}.
    \end{align*}
    where $B\defeq \max_{m=1,\dots,M}\{B_m\}$.
    The above constructions and calculations immediately imply the conclusions of the lemma.
\end{proof}

\subsection{Proof of Theorem \ref{approximation theorem}} \label{subsection: 5.3}

\begin{proof}
    By Lemma \ref{lemma_approx}, for any $\delta$ and given $p$, there exists a set $\Omega_0$ with $P_X$-measure at least $1-\delta$ on which there is a neural network $I_p$ with $O(p)$ neurons and weights such that $\norm{f_{\phi}^* - I}_{\Omega_0, \infty} \leq C p^{-1/2}$. 
    Now we can write
    \begin{align*}
        &L_{\phi}(I_p) - L_{\phi}(f_{\phi}^*)\\ 
        &= E[\eta(X)(\phi(I_p(X)) - \phi(f_{\phi}^*(X)))\\
        & \qquad + (1-\eta(X))(\phi(-I_p(X)) - \phi(-f_{\phi}^*(X)))]\\
        &\leq E[\eta(X)|\phi(I_p(X)) - \phi(f_{\phi}^*(X))|\\
        & \qquad + (1-\eta(X))|\phi(-I_p(X)) - \phi(-f_{\phi}^*(X))|] \\
        &\leq E[|I_p(X) - f_{\phi}^*(X)|]\\
        &\leq \norm{\mathbbm{1}_{\Omega_0}(X)( I_p(X) - f_{\phi}^*(X))}_{\infty} + O(\delta)\\
        &\leq Cp^{-1/2}
    \end{align*}
    where the last inequality follows by choosing $\delta=O(p^{-1/2})$.
\end{proof}

\subsection{Proof of Proposition \ref{approximate excess risk rate}}\label{subsection: 5.4}
Here we give a proof of Proposition \ref{approximate excess risk rate} that constitutes Step 1 of our analysis mentioned in Section \ref{subsection: 4.3}.\\
\textbf{Step 1}: First, we state some standard results from empirical process theory often used in the ERM literature.
Define an envelop of a function class $\mathcal{F}$ as a measurable function $F$ such that
\begin{align*}
    f(x) \leq F(x),\quad \forall x, \forall f\in \mathcal{F}.
\end{align*}
The following result from \cite{gine2006concentration} (see Theorem 3.1, Example 3.5) provides a bound on the Rademacher complexity in terms of bound on the covering number:

\begin{proposition}\label{rademacher_bound}
    For a class of functions $\mathcal{F}$ uniformly bounded by $U$ and an envelop $F$ and empirical distribution $(P_X)_n$, suppose for some $v>0$ and some $A>0$, the following holds for all $\omega$ in the underlying probability space:
    \begin{align}\label{eq:covering_number}
        N(\epsilon, \mathcal{F}, L_2((P_X)_n)) \leq \left( A\frac{\norm{F}_{L_2((P_X)_n)}}{\epsilon} \right)^v.
    \end{align}
    If we let $\sigma^2 \defeq \sup_{f\in\mathcal{F}}Pf^2$
    then, for some universal constant $C>0$
    \begin{align*}
    E\left[ \norm{R_n}_{\mathcal{F}} \right] \leq C\max\biggl\{\sqrt{\frac{v}{n}} \sigma \sqrt{\log \frac{A\|F\|_{L_{2}(P_X)}}{\sigma}},\\
    \frac{vU}{n} \log \frac{A\|F\|_{L_{2}(P_X)}}{\sigma} \biggr\}.
    \end{align*}
\end{proposition}
Furthermore, for any VC-subgraph class of functions $\mathcal{F}$ with VC-index $V(\mathcal{F})$, we have the following standard result from \cite{wellner2013weak} (Theorem 2.6.7):
\begin{proposition}\label{covering_number_bound}
    For a VC-class of functions with measurable envelope function $F$ and $r \geq 1$, one has for any probability measure $Q$ with $\|F\|_{Q, r}\defeq (\int F^r dQ)^{1/r}>0$ where,
$$
N\left(\varepsilon\|F\|_{Q, r}, \mathcal{F}, L_{r}(Q)\right) \leq K V(\mathcal{F})e^{V(\mathcal{F})}\left(\frac{1}{\varepsilon}\right)^{r(V(\mathcal{F})-1)},
$$
for a universal constant $K$ and $0<\varepsilon<1$.
\end{proposition}
Note that the above proposition is a universal result in the sense that it holds for any choice of probability measure $Q$. Thus, applying Proposition \ref{covering_number_bound} with any realization (by $\omega$) of probability measure $(P_X)_n(\omega)$ and $r=2$, we have that
\begin{align}\label{covering number bound}
N\left(\varepsilon, \mathcal{F}, L_{2}((P_X)_n)\right) \leq K V(\mathcal{F})(16 e)^{V(\mathcal{F})}\left(\frac{\|F\|_{Q, r}}{\varepsilon}\right)^{2(V(\mathcal{F})-1)}. 
\end{align}
Thus we have that any VC-class of functions $\mathcal{F}$ with VC-index $V(\mathcal{F})$ satisfies \eqref{eq:covering_number} with $v=2(V(\mathcal{F})-1)$.




Another lemma we will use is the following bound on the VC-index of the class of feed-forward ReLU networks shown in \cite{bartlett2019nearly}:
\begin{lemma}\label{vc bound}
    The VC-index of the class of feed-forward ReLU networks with $W$ total number of parameters and $L$ number of layers satisfies
    \begin{align*}
        c_1WL \log(W/L)\leq V(\mathcal{F})\leq c_0WL\log(W)
    \end{align*}
    for some constants $c_0,c_1>0$.
\end{lemma}

Suppose that $\mathcal{G}_n$ is the class of feed-forward  ReLU neural networks with $W(n)$ total number of parameters and a constant number of layers (at least 3) that take values in $[-M/2, M/2]$. Because the entropy bound \eqref{covering number bound} holds, we can apply Proposition \ref{rademacher_bound} with $\mathcal{F} = \{g-\Tilde{g}: g\in \mathcal{G}_n, \norm{g-\Tilde{g}}_{L_2(P_X)}^2\leq\delta \}$ to conclude that 
\begin{align*}
    \omega_n(\delta) \leq C\max\biggl\{\sqrt{\frac{2(V(\mathcal{G}_n)-1)}{n}} \sqrt{\delta} \sqrt{\log \frac{M}{\sqrt{\delta}}},\\ \frac{2(V(\mathcal{G}_n)-1)M}{n} \log \frac{M}{\sqrt{\delta}} \biggr\}. 
\end{align*}
This implies that 
\begin{align}
    \omega_n^{\sharp}(c) \leq \frac{CV(\mathcal{G}_n)}{nc^2}\log\left( \frac{Mnc^2}{V(\mathcal{G}_n)} \right) \label{eq73}
\end{align}
for appropriately redefined constant $C>0$.

Then, Lemma \ref{vc bound} implies that there exists some $c_0,c_1>0$ such that 
\begin{align}\label{eq:10}
c_1W(n)\leq V(\mathcal{G}_n)\leq c_0W(n)\log(W(n)). 
\end{align}
Combining \eqref{eq73} and \eqref{eq:10}, we immediately get
\begin{align} 
    \omega_n^{\sharp}(c) \leq \frac{Cc_0W(n)\log(W(n))}{nc^2}\log\left( \frac{Mnc^2}{c_1W(n)} \right) . \label{eq75}
\end{align}
 

Now we analyze the term $\max\left\{\omega_n^{\sharp}\left(c\alpha\right)-\tau_n, \tau_n\alpha\right\}$ that appears in the excess risk bound of Theorem \ref{theo1}. The result of Theorem \ref{approximation theorem} implies that we may take $\tau_n = \frac{Cd}{\sqrt{W(n)}}$.

Because of the freedom of choosing $\alpha>0$, we may assume $\alpha=n^{-u}$ for some $u\geq 0$. Furthermore, put $W(n) = n^r$ for $r>0$. With these substitutions and from \eqref{eq75}, we get for appropriately redefined constants $C, c>0$, the following:
\begin{align*}
    \omega_n^{\sharp}\left(c\alpha\right)-\tau_n &\leq C n^{r+2u-1}\log(n^r)\log(cn^{1-2u-r}),\\
    \tau_n\alpha &\leq Cn^{-u-r/2}.
\end{align*}

Plugging in the above display to the excess risk bound of Theorem \ref{theo1}, we can conclude that for some constant $K,C>0$, we have
\begin{align*}
    P\biggl(\widehat{\mathcal{E}}_{\phi}(\widehat{g}_n) \geq &K \biggl( \max \left\{ n^{r+2u-1}\log n^r\log(n^{1-2u-r}), n^{\frac{-2u-r}{2}} \right\}\\
    &+ \frac{t}{n} + \sqrt{\frac{t}{n^{1+r/2}}} \biggr)\biggr) \leq Ce^{-t}.
\end{align*}
Since $ n^{r+2u-1} \leq n^{-u-r/2}$ if and only if $r+2u-1 \leq -u-r/2$, we get that $\max \left\{ n^{r+2u-1}, n^{-u-r/2} \right\} = n^{-u-r/2}$ if and only if $r\leq -2u+2/3$. In this case, by choosing $r=2/3,u=0$, we get for some constant $K>0$,
\begin{align*}
    P\left(\widehat{\mathcal{E}}_{\phi}(\widehat{g}_n) \geq K \left( 1+t \right)n^{-1/3}\log(n) \right) \leq Ce^{-t}.
\end{align*}
It is easy to check that the other possible scenario of $ n^{r+2u-1} \geq n^{-u-r/2}$ leads to the same rate after computations similar to the above.

\subsection{Proof of Theorem \ref{minimax lower bound theorem}}\label{subsection: 5.6}
We provide lower bounds on the following minimax excess risk for classification:
\begin{align}\label{minimax excess risk}
     \inf_{M \text{ measurable}}\sup_{P\in \Sigma}P(M(\mathbf{X})\neq Y) - P(M^*(\mathbf{X})\neq Y)
\end{align}
where the supremum is taken over the class of distributions satisfying Assumption \ref{assumptions on distribution}.
We start by stating and proving a result that holds under a more general situation. The construction closely follows that found in \cite{audibert2007fast} and is based on an application of Assouad's lemma, which can be found for example in \cite{tsybakov2004introduction}.

Consider the partition of $[0,1]^d$ by the canonical grid: For positive integer $q$, define the following finite subset of $\mathbb{R}^d$

\begin{align}\label{grid}
    G_q &= \biggl\{\biggl(\frac{2k_1+1}{2q},\dots,\frac{2k_d+1}{2q} \biggr): \nonumber\\
    &\qquad \qquad k_i=0,\dots,q-1,i=1,\dots,d \biggr\} . 
\end{align}
There are $q^d$ elements in $G_q$, and we can label them in some order, for example by dictionary order based on the coordinate and magnitude of $k_i$. Let $g_1,\dots,g_{q^d}$ be such labeled points. For any $x\in [0,1]^d$ denote by $g(x)\in G_q$ the point in $G_q$ closest to $x$ so that $g(x) = \argmin_{g \in G_q} \norm{x-g}_2$ where argmin is well-defined with appropriate tie-breaking rule. Then we can write the partition as follows:
\begin{align*}
    M_i = \{x\in[0,1]^d: g(x) = g_i \},
\end{align*}
\begin{align*}
    [0,1]^d = \bigcup_{i=1}^{q^d} M_i
\end{align*}
where now $[0,1]^d$ is the union of $q^d$ disjoint sets of same Lebesgue measure.

Now we build up a setting appropriate for applying Assouad's lemma to obtain a minimax lower bound. To that end, we first define the following finite class of probability distributions: Choose a positive integer $m\leq q^d$ and define the class of distributions $\mathcal{H} \defeq \{P_{\mathbf{\sigma}}: \mathbf{\sigma} = (\sigma_1, \dots, \sigma_m) \in \{0,1\}^m\}$ where each $P_{\sigma}$ represents a distinct distribution of $(\mathbf{X}, Y)$ on $\mathbb{R}^d \times \{0,1\}$. Now we define each $P_{\sigma}$ by providing the marginal distribution for $\mathbf{X}$ and the conditional distribution $P(Y=1|X)$.

We shall define the marginal distribution of $\mathbf{X}$, denoted by $(P_{\sigma})_X$ in the same way for every choice of $\sigma$. Let $w$ be some real number satisfying $0<w<1/m$ and write $B(x, r)$ to mean the Euclidean ball in $\mathbb{R}^d$ centered at $x\in \mathbb{R}^d$ and radius $r>0$. Define $A$ as some bounded subset of $\mathbb{R}^n \backslash \bigcup_{i=1}^m M_i$.  We then define  $(P_{\sigma})_X$ to be the measure absolutely continuous with respect to the $d$-dimensional Lebesgue measure $\lambda$ with density $u$ (nonnegative, integrable function on $\mathbb{R}^d$) defined as
\begin{align*}
    u(x) = 
    \begin{cases}
        \frac{w}{\lambda
        (B(0,1/(4q)))}, & \text{ if } x \in B\left(g,\frac{1}{4q}\right) \text{ for some } g \in G_q; \\
        \frac{1-mw}{\lambda(A)}, & \text{ if } x \in A;\\
        0, & \text{ otherwise.}
    \end{cases}
\end{align*}
In words, the marginal distribution of $\mathbf{X}$ is supported on balls centered at the $m$ points in the grid $G_q$, actually with constant density, and the bounded set $A$.
Now, we define the Borel-measurable function $\eta_{\sigma}:\mathbb{R}^d \rightarrow [0,1]$ such that $\eta_{\sigma}(\mathbf{X})$ is a version of $P_{\sigma}(Y=1|\mathbf{X})$ (see \cite[Theorem 9.1.3]{chung2001course} for existence of such a function).

Let $h:[0,\infty) \rightarrow [0,\infty)$ be the nonincreasing, infinitely differentiable function defined as:
\begin{align*}
    h(x) &= \int_{x}^{1/2} h_1(t)dt\bigg/\left(\int_{1/4}^{1/2} h_1(t)dt\right) , \\
    h_1(x) &= 
    \begin{cases}
        \exp\left(-\frac{1}{(1/2-t)(t-1/4)}\right), \text{ for } t \in (1/4,1/2); \\
        0 , \text{ for } t \in [0,1/4]\cup [1/2,\infty) .
    \end{cases}
\end{align*}
Note $h=1$ on $[0,1/4]$ and $h=0$ on $[1/2,\infty)$. Then, we define $\phi :\mathbb{R}^d \rightarrow [0,\infty)$ as $\phi(x) \defeq q^{-r}h(\norm{x}_2)$ for some $r>0$ to be specifed later. Observe that $\phi$ is an infinitely differentiable bump function supported in $[0,1/2)$ and as shown in \cite[Section IX]{barron1993universal}, is a Barron function. In particular, it is an element of $\mathcal{BA}(\Omega)$. 

For the same $q$ used in defining the grid $G_q$ in \eqref{grid}, we finally define for an arbitrary $\sigma \in \{0,1\}^m$
\begin{align*}
    \eta_{\sigma}(x) = 
    \begin{cases}
        \frac{1+\sigma_i \phi(q(x-g(x)))}{2}, \text{ for } x \in M_i, \quad i=1,\dots,m; \\
        1/2, \text{ for }  x \in A; \\
        0, \text{ otherwise.}
    \end{cases}
\end{align*}

We first verify that all $P_{\sigma}\in\mathcal{H}$ satisfies the margin assumption: for any fixed choice of $x_0 \in G_q$,
\begin{align*}
    &P_{\sigma}(0<|\eta_{\sigma}-1/2|\leq t)\\
    &= P_{\sigma}(0< \phi(q(x-g_x)) \leq 2t)\\
    &= m\int_{B(x_0,1/(4q))}\mathbbm{1}(0<\phi(q(x-x_0))\leq 2t)u(x) dx\\
    &= m\int_{B(0,1/4)}\mathbbm{1}(0<\phi(x)\leq 2t)\frac{w}{q^{d}\lambda(B(0,1/(4q)))} dx\\
    &= mw\mathbbm{1}(t\geq q^{-r}/2).
\end{align*}
Thus if the choice of $m,w$ is such that $mw\leq (q^{-r}/2)^{\alpha}$, the margin assumption is satisfied.

In order to apply Assouad's lemma to obtain a minimax lower bound, it is first necessary to relate the minimax excess risk of \eqref{minimax excess risk} to the minimax risk for the Hamming distance between $\sigma$'s used to define $\mathcal{H}$. Precisely, define $\rho(\sigma,\sigma^{\prime})$ as the Hamming distance between $\sigma$ and $\sigma^{\prime}$: $\rho: \{0,1\}^m \times \{0,1\}^m \rightarrow \{0,1,\dots,m\}$ such that $\rho(\sigma,\sigma^{\prime})$ equals the number of positions in which $\sigma$ and $\sigma^{\prime}$ differ. Then, for any classifier $\widehat{f}_n$, we want to have the following bound that holds uniformly over all such $\widehat{f}_n$: 
\begin{align*}
    &\sup_{P\in \Sigma}E[P(\widehat{f}_n(\mathbf{X})\neq Y)] - P(f^*(\mathbf{X})\neq Y)\\
    &\geq \inf_{\widehat{\sigma}}\max_{\sigma \in \{0,1\}^m} E_{P_{\sigma}^n}[\rho(\widehat{\sigma},\sigma)]
\end{align*}
where $f^*$ is the Bayes classifier for $P$. 

We proceed as follows: If we denote by $f^*_{\sigma}$ the Bayes classifier for measure $P_{\sigma}$, we can write
\begin{align} 
    &P_{\sigma}(\widehat{f}(\mathbf{X})\neq Y) - P_{\sigma}(f^*(\mathbf{X})\neq Y) \nonumber\\ 
    &= 2\left[\int \left|\eta_{\sigma}(x)-\frac{1}{2}\right|\mathbbm{1}(\widehat{f}_n(x) \neq f_{\sigma}^*({x})) (P_{\sigma})_X(dx) \right] \nonumber\\
    &= \sum_{i=1}^m \int_{M_i}\left|\phi(q(x-g(x)))\right| \mathbbm{1}(\widehat{f}_n(x) \neq f_{\sigma}^*({x}))(P_{\sigma})_X(dx) \label{eq15}
\end{align}
where first equality follows from the standard formula for excess risk found for example in \cite[Theorem 2.2]{devroye2013probabilistic} and second equality follows from our construction in the preceding paragraphs.

Now if we define 
\begin{align*}
    \widehat{\sigma}_i \defeq \argmin_{t=0,1} \int_{M_i} \left|\phi(q(x-g(x)))\right|\mathbbm{1}(\widehat{f}_n(x) \neq t)(P_{\sigma})_X(dx)
\end{align*}
and observe that for all $x$ in each $M_i$, 
\begin{align*}
    f^*_{\sigma}(x) =
    \begin{cases}
        1, \text{ if } \sigma_i = 1;\\
        0, \text{ if } \sigma_i = 0,
    \end{cases}
\end{align*}
we can lower bound \eqref{eq15} as
\begin{align*}
    &\eqref{eq15}\\
    &\geq  \frac{1}{2}\sum_{i=1}^m \int_{M_i}\left|\phi(q(x-g(x)))\right|\left|\sigma_i -\widehat{\sigma}_i\right|(P_{\sigma})_X(dx)\\
    &= \frac{1}{2}\sum_{i=1}^m \int_{B(g_i,1/(4q))}\left|\phi(q(x-g_i))\right|\left|\sigma_i -\widehat{\sigma}_i\right|\\
    &\qquad \qquad \qquad \qquad \times \frac{w}{\lambda(B(0,1/(4q)))}dx\\
    &= \frac{1}{2}\sum_{i=1}^m \int_{B(0,1/4)}\left|\phi(x)\right|\left|\sigma_i -\widehat{\sigma}_i\right|\frac{w}{q^d\lambda(B(0,1/(4q)))}dx\\
    &=\frac{1}{2}\norm{\sigma-\widehat{\sigma}}_{1}  \frac{w}{\lambda (B(0,1/4))}\int_{B(0,1/4)}\left|\phi(x)\right|dx\\
    &=\frac{q^{-r}w}{2}\norm{\sigma-\widehat{\sigma}}_{1} .
\end{align*}

Hence, 
\begin{align}
    &\sup_{P\in \Sigma}E_{P^n}P(\widehat{f}_n(\mathbf{X})\neq Y) - P(f^*(\mathbf{X})\neq Y) \nonumber\\
    &\geq \max_{P_{\sigma} \in \mathcal{H}} E_{P_{\sigma}^n}\left[P_{\sigma}(\widehat{f}_n(\mathbf{X})\neq Y) - P_{\sigma}(f_{\sigma}^*(\mathbf{X})\neq Y)\right] \nonumber\\
    &\geq \frac{q^{-r}w}{2}\max_{P_{\sigma\in \mathcal{H}}}E_{P_{\sigma}^n} \left[ \norm{\sigma-\widehat{\sigma}}_{1}\right] . \label{24}
\end{align}

Let $V(P,Q)$ denote the total variation distance between two probability measures $P$ and $Q$. Now we can apply a version of Assouad's lemma from \cite[Theorem 2.12]{tsybakov2004introduction} which states the following:
\begin{lemma}[Assouad's Lemma]
    If $V\left(P_{\sigma^{\prime}}, P_\sigma\right) \leq \alpha<1, \quad \forall \sigma, \sigma^{\prime} \in \{0,1\}^m$ such that their Hamming distance $\rho\left(\sigma, \sigma^{\prime}\right)=1$, then
    \begin{align*}
        \inf _{\hat{\sigma}} \max _{\sigma \in \{0,1\}^m} E_{P_\sigma} \rho(\hat{\sigma}, \sigma) \geq \frac{m}{2}(1-\alpha ) .
    \end{align*}
\end{lemma}
We now use the fact that the total variation distance can be upper bounded by the Hellinger distance we denote by $H(P,Q)$. That is, $V(P_{\sigma}^n, P_{\sigma'}^n) \leq H(P_{\sigma}^n,P_{\sigma'}^n)$. One property of squared Hellinger distance between two $n$-product measures is subadditivity:
\begin{align*}
    H^2(P_{\sigma}^n, P_{\sigma'}^n) \leq \sum_{i=1}^n H^2(P_{\sigma}, P_{\sigma'}) = nH^2(P_{\sigma}, P_{\sigma'}) .
\end{align*}
Combining this with Assouad's lemma, we have 
\begin{align}
    \inf _{\hat{\sigma}} \max _{\sigma \in \{0,1\}^m} E_{P_\sigma} \rho(\hat{\sigma}, \sigma) \geq \frac{m}{2}(1-\sqrt{n}H(P_{\sigma}, P_{\sigma'})) . \label{27}
\end{align}
Now, the Hellinger distance $H(P_{\sigma}, P_{\sigma'})$ when $\rho(\sigma,\sigma')=1$ and $i$ is the index position on which $\sigma$ and $\sigma'$ differ can be easily computed as
\begin{align}
    &H^2(P_{\sigma}, P_{\sigma'}) \nonumber\\
    &= \frac{2w}{\lambda
        (B(0, \frac{1}{4q}))} \int_{B(g_i, \frac{1}{4q})} \biggl(\sqrt{\left(\frac{1+\phi(q(x-g(x)))}{2}\right)} \nonumber\\ 
    &\qquad \qquad \qquad - \sqrt{\left(\frac{1- \phi(q(x-g(x)))}{2}\right)}\biggr)^2dx \nonumber \\
    &= \frac{2w}{\lambda
        (B(0,\frac{1}{4}))}\int_{B(0,\frac{1}{4})} \biggl(\sqrt{\left(\frac{1+\phi(x))}{2}\right)} \\
    &\qquad \qquad \qquad - \sqrt{\left(\frac{1- \phi(x)}{2}\right)}\biggr)^2 dx \nonumber \\
    &= \frac{2w}{\lambda
        (B(0,\frac{1}{4}))}\int_{B(0,\frac{1}{4})}1 - \sqrt{1-\phi^2(x)} dx \nonumber\\
    &= 2w(1- \sqrt{1-q^{-2r}}) \nonumber \\
    &\leq 2wq^{-2r} . \label{30}
\end{align}
By combining \eqref{24}, \eqref{27}, \eqref{30}, we obtain the first result:
\begin{align}\label{eq54}
    &\inf_{\widehat{f}_n}\sup_{P\in \Sigma}E\left[P(\widehat{f}_n(\mathbf{X})\neq Y) - P(M^*(\mathbf{X})\neq Y) \right] \nonumber \\
    &\geq Cq^{-r}mw(1-q^{-r}\sqrt{nw})
\end{align}
for some constant $C>0$, and $m\leq q^d$, $w \leq 1/m$ $wm\leq q^{-r\alpha}/2^{\alpha}$.
Now make the following choice for $m, w, r, q$:
\begin{align*}
    m &= q^d,\\
    w &= \frac{q^{-\alpha r-d}}{2^\alpha},\\
    r &= \frac{2d}{2+\alpha},\\
    q &= \lfloor \overline{C} n^{\frac{1}{3r(2+\alpha)}}\rfloor.
\end{align*}
for appropriate constant $\overline{C}\geq 2$ whose choice will be soon specified.
We verify that for the above choice of parameters, for any $\sigma$, $P_{\sigma}\in \mathcal{H}$ satisfies all four points in Assumption \ref{assumptions on distribution}. First, as we already mentioned, the margin condition is satisfied if $mw \leq \frac{q^{-r\alpha}}{2^{\alpha}}$, which is true by construction. Second, observe that by definition, for any $\sigma$, $\eta_{\sigma}(x)$ is bounded away from $0$ and $1$ by $\frac{1}{2} - \frac{q^{-d}}{2}>0$, again by construction. Third, it is clear that $X$ is compactly supported in our setup. Fourth, since $\eta_{\sigma}$ is defined as an infinitely differentiable function on each $M_i$, $i=1,\dots,m$ and constant on $A$, $\eta|_{M_i}$ and $\eta_{A}$ all belong to the Barron approximation space, as required.

Finally, we need to verify that the right-hand side of \eqref{eq54} yields the desired rate. First, to check that $1-q^{-r}\sqrt{nw}$ is indeed positive, it sufficies to check $q^r w^{-\frac{1}{2}} \geq \sqrt{n}$. Since $q^r w^{-\frac{1}{2}} = q^{\frac{(2+\alpha)r+d}{2}}/2^{\frac{\alpha}{2}} =\lfloor \frac{\overline{C}}{2^{\alpha/2}} n^{\frac{(2+\alpha)r+d}{3r(2+\alpha)}}\rfloor$, by choosing $\overline{C}\geq 2^{\alpha/2}$ we only need to choose $r$ such that $\frac{2d}{r}\geq 2+\alpha$. In particular, $r= \frac{2d}{2+\alpha}$ works. This computation also shows that $1-q^{-r}\sqrt{nw}$ is a constant depending on the choice of $\overline{C}$ and $\alpha$. This implies the rate in \eqref{eq54} is determined by $q^{-r}mw = q^{-(\alpha+1)r}/2^{\alpha}= \lfloor \frac{\overline{C}}{2^{\alpha}} n^{-\frac{1+\alpha}{3(2+\alpha)}} \rfloor$, which is indeed the rate in Theorem \ref{minimax lower bound theorem}. 

\section{Conclusion}
This work has derived a new, non-asymptotic rate of convergence for the excess risk when the classifier is the empirical risk minimizer of the logistic loss. The class of distributions studied is characterized by the Barron approximation space, which includes the classical Barron functions as a proper subset. This regime is particularly interesting for neural networks precisely because they achieve dimension-free approximation rates here. A matching lower bound for the minimax rate of convergence is also derived, showing the minimax optimality of the proposed neural network-based classifiers.\\
Our results suggest a rich avenue for future research: what happens when the regression functions belong to other classical function spaces from approximation theory such as $L^2$-Sobolev space, cartoon functions, and bounded variation functions? It is shown in \cite{elbrachter2021deep} that neural networks are indeed Kolmogorov-Donoho optimal approximants of many of these spaces. One important point in this result is that Kolmogorov-Donoho optimal rate can be achieved only when the architecture of the neural network is deep: specifically, the depth of the network has to scale polylogarithmically to the inverse of the desired approximation accuracy. This will obviously be a regime different from ours and more general in the sense that now depth matters. We hope to study how these results can be translated into excess risk convergence results in classification context in our future work.

\printbibliography
\end{document}